\documentclass[accepted]{uai2025} 
\usepackage[T1]{fontenc}
\usepackage[utf8]{inputenc} 

\usepackage{subfig}
\usepackage{caption}

\usepackage{microtype}
\usepackage{graphicx}        
\usepackage{multirow}
\usepackage[table,xcdraw]{xcolor}
\usepackage{booktabs} 
\usepackage{amsmath}
\DeclareMathOperator*{\argmax}{arg\,max} 
\usepackage[american]{babel}
\usepackage{hyperref}

\usepackage{algorithm}
\usepackage{algorithmic}

\usepackage{natbib} 
\usepackage{mathtools} 
\usepackage{booktabs} 
\usepackage{tikz} 

\usepackage{amsmath}
\usepackage{amssymb}
\usepackage{mathtools}
\usepackage{amsthm}

\usepackage{enumitem}

\usepackage[capitalize,noabbrev]{cleveref}

\theoremstyle{plain}
\newtheorem{theorem}{Theorem}[section]
\newtheorem{proposition}[theorem]{Proposition}
\newtheorem{lemma}[theorem]{Lemma}

\theoremstyle{definition}
\newtheorem{definition}[theorem]{Definition}
\newtheorem{assumption}[theorem]{Assumption}
\theoremstyle{remark}
\newtheorem{remark}[theorem]{Remark}

\newcommand{\cX}{\mathcal{X}}
\newcommand{\cU}{\mathcal{U}}
\newcommand{\RR}{\mathbb{R}}
\newcommand{\EE}{\mathbb{E}}
\newcommand{\cN}{\mathcal{N}}

\newcommand{\cD}{\mathcal{D}}
\newcommand{\cP}{\mathcal{P}}
\newcommand{\cF}{\mathcal{F}}
\newcommand{\cR}{\mathcal{R}}
\newcommand{\cZ}{\mathcal{Z}}
\newcommand{\cQ}{\mathcal{Q}}
\newcommand{\cT}{\mathcal{T}}
\newcommand{\cS}{\mathcal{S}}
\newcommand{\algname}{\text{PURE}}
\newcommand{\algbase}{\text{PURE}_{\text{base}}}
\newcommand{\algpolicy}{\text{PURE}_{\text{LowSwitch}}}
\newcommand{\algmss}{\text{PURE}_{\text{LowRollout}}}
\newcommand{\algseiko}{\algname_{\text{SEIKO}}}
\newcommand{\algenode}{\algname_{\text{ENODE}}}

\title{Sample and Computationally Efficient Continuous-Time Reinforcement Learning with General Function Approximation}

\author[1]{\href{mailto:<zhaorunz@iu.edu>}{Runze Zhao}\thanks{These authors contributed equally to this work.}}
\author[2]{\href{mailto:<yyu3@iu.edu>}{Yue Yu}\footnote[1]}
\author[3]{\href{mailto:<yzhu1221@umd.edu>}{Adams Yiyue Zhu}}
\author[1]{\href{mailto:<cya2@iu.edu>}{Chen Yang}}
\author[1]{\href{mailto:<dz13@iu.edu>}{Dongruo Zhou}\thanks{Corresponding author.}}

\affil[1]{%
    Luddy School of Informatics, Computing, and Engineering\\
    Indiana University Bloomington\\
    Bloomington, Indiana, USA
}

\affil[2]{%
    Department of Statistics\\
    Indiana University Bloomington\\
    Bloomington, Indiana, USA
}

\affil[3]{%
    Department of Electronic and Computer Engineering\\
    University of Maryland, College Park\\
    College Park, Maryland, USA
}
\begin{document}
\maketitle

\begin{abstract}
Continuous-time reinforcement learning (CTRL) provides a principled framework for sequential decision-making in environments where interactions evolve continuously over time. Despite its empirical success, the theoretical understanding of CTRL remains limited, especially in settings with general function approximation. In this work, we propose a model-based CTRL algorithm that achieves both sample and computational efficiency. Our approach leverages optimism-based confidence sets to establish the first sample complexity guarantee for CTRL with general function approximation, showing that a near-optimal policy can be learned with a suboptimality gap of $\tilde{O}(\sqrt{d_{\mathcal{R}} + d_{\mathcal{F}}}N^{-1/2})$ using $N$ measurements, where $d_{\mathcal{R}}$ and $d_{\mathcal{F}}$ denote the distributional Eluder dimensions of the reward and dynamic functions, respectively, capturing the complexity of general function approximation in reinforcement learning. Moreover, we introduce structured policy updates and an alternative measurement strategy that significantly reduce the number of policy updates and rollouts while maintaining competitive sample efficiency. We implemented experiments to backup our proposed algorithms  on continuous control tasks and diffusion model fine-tuning, demonstrating comparable performance with significantly fewer policy updates and rollouts. The code is available at \url{https://github.com/MLIUB/PURE}.
\end{abstract}

\section{Introduction}
Continuous-time reinforcement learning (CTRL) is a fundamental problem in learning-based control, with numerous applications in robotics, finance, healthcare, and autonomous systems. Many real-world decision-making problems are more naturally modeled in continuous time rather than discrete time, as they involve continuous interaction with the environment. The goal of CTRL is to find an optimal policy that continuously interacts with and adapts to the environment to maximize long-term rewards. A growing body of work has demonstrated the empirical success of CTRL, leveraging approaches such as model-based continuous-time control \citep{greydanus2019hamiltonian, yildiz2021continuous, lutter2021value, treven2024efficient} and fine-tuning in diffusion models \citep{yoon2024censored, xie2023difffit}. These studies have shown promising results in real-world tasks, indicating that continuous-time policies can outperform their discrete-time counterparts in various applications.

Despite these empirical advances, the theoretical understanding of CTRL remains limited. A fundamental question in CTRL is \textit{sample efficiency}, which refers to the total number of measurements an agent must take from the environment to learn a near-optimal policy. Existing works have primarily focused on specific settings, such as linear quadratic regulators (LQR) \citep{cohen2018online, abeille2020efficient, simchowitz2020naive} or well-calibrated statistical models \citep{treven2024efficient}, where strong structural assumptions facilitate theoretical analysis. However, this stands in contrast to empirical practice, where general function classes—such as neural networks—are widely used. These structured models often fail to capture the complexity of real-world environments, highlighting the need for a more general theoretical framework. Thus, we pose the following question:

\begin{center}
    \textit{What is the sample complexity for CTRL with general function approximation to find a near-optimal policy?}
\end{center}

Beyond sample efficiency, another crucial aspect is \textit{computational efficiency}, which is characterized by minimizing the number of policy updates and episode rollouts during the online learning phase. Unlike discrete-time RL, where sample complexity is tightly coupled with the number of episode rollouts, CTRL allows for multiple—even an infinite number of—measurements within a single rollout. This flexibility enables practitioners to employ various measurement strategies, such as equidistant sampling or adaptive strategies, to enhance learning efficiency. While empirical studies have explored multiple measurement strategies \citep{treven2024efficient}, their theoretical understanding remains limited, particularly regarding the tradeoff between computational efficiency and sample complexity. This leads to our second fundamental question:

\begin{center}
    \textit{Can we develop provable new measurement strategies that enhance computational efficiency without significantly sacrificing sample efficiency?}
\end{center}

In this work, we answer the above questions affirmatively. Specifically, we study model-based CTRL in a general function approximation setting, where both the dynamic model and policies are approximated using a general function class. We propose an algorithm, \textit{Policy Update and Rolling-out Efficient CTRL} ($\algname$), which achieves both sample and computational efficiency. Our main contributions are as follows:

\begin{enumerate}[leftmargin = *]
    \item We first introduce $\algbase$, the foundational version of $\algname$, which focuses on sample efficiency using the optimism-in-the-face-of-uncertainty principle \citep{abbasi2011improved} and a confidence set construction for the underlying environment. Theoretically, we prove that with $N$ measurements, $\algbase$ finds a near-optimal policy with a suboptimality gap of $\tilde{O}(\sqrt{d_{\cR} + d_{\cF}}N^{-1/2})$, where $d_{\cR}$ and $d_{\cF}$ denote the Eluder dimensions \citep{russo2013eluder, jin2021bellman} of the reward and dynamic functions, respectively. This result provides the first known sample complexity guarantee for CTRL with general function approximation. Notably, unlike prior works such as \citet{treven2024efficient}, $\algbase$ does not rely on an external calibration model, which often requires strong smoothness assumptions that are difficult to satisfy and verify in practice.

    \item To improve computational efficiency, we propose $\algpolicy$, an extension of $\algbase$ that incorporates a tailored policy update strategy, reducing the number of policy updates from $N$ to $O(\log N(d_{\cR} + d_{\cF}))$. This represents a significant reduction for many function classes. Furthermore, we introduce $\algmss$, designed to minimize the number of policy rollouts. We prove that $\algmss$ reduces the number of rollouts by a factor of $m$, achieving a suboptimality gap of $\tilde{O}(\sqrt{C_{\cT, m}}N^{-1/2} + m/N)$, where $C_{\cT,m}$ is the \emph{independency coefficient} that used for quantifying the independency between each measurement. Our results suggest that one can further improve computational complexity for CTRL.

    \item We empirically backed up our theoretical findings by implementing $\algname$ in both the traditional continuous-time RL framework \citep{yildiz2021continuous} and the diffusion-model fine-tuning framework \citep{uehara2024feedback}. Our experimental results demonstrate the practical advantages of $\algname$, achieving comparable performance with fewer policy updates and rollouts.
\end{enumerate}

\section{Related Works}

\paragraph{Continuous-Time Reinforcement Learning}  
Our algorithms fall into Continuous-Time Reinforcement Learning (CTRL), which has been extensively studied by the control community for decades \citep{doya2000reinforcement, vrabie2009neural}, primarily focusing on planning or simplified models such as the linear quadratic regulator \citep{pmlr-v195-shirani-faradonbeh23a,doi:10.1137/17M1152280,huang2024sublinear,JMLR:v23:20-664,doi:10.1137/22M1515744}. A significant shift occurred with \citet{chen2018neural}, which introduced CTRL with nonlinear function approximation, enabling continuous-time representations to be learned using neural networks. Building on this foundation, \citet{yildiz2021continuous} proposed an episodic model-based approach that iteratively fits an ODE model to observed trajectories and solves an optimal control problem via a continuous-time actor-critic method. More recently, \citet{holt2024active} investigated CTRL under a costly observation model, demonstrating that uniform time sampling is not necessarily optimal and that more flexible sampling policies can yield higher returns. While these works primarily focus on empirical studies of CTRL with nonlinear function approximation, theoretical understanding remains limited. In this direction, \citet{treven2024efficient} analyzed deterministic CTRL with nonlinear function approximation, introducing the concept of a \emph{measurement selection strategy} (MSS) to adaptively determine when to observe the continuous state for optimal exploration. Extending this line of research, \citet{treven2024sense} studied stochastic CTRL under a cost model, aiming to minimize the number of environment observations. Our work builds upon \citet{treven2024efficient} by considering a broader function approximation class and providing theoretical insights into the tradeoff between sample efficiency and computational efficiency, without relying on strong assumptions about the epistemic uncertainty estimator.

\paragraph{Reinforcement Learning with Low Switching Cost}

In many real-world RL applications, frequently updating the policy can be impractical or costly. This has motivated the study of low-switching RL, where the agent deliberately restricts how often its policy changes. Early works focus on the bandit setting, including multi-armed bandits \citep{auer2002finite, cesa2013online, gao2021provably} and linear bandits \citep{abbasi2011improved, ruan2021linear}, among others. In the RL setting, \citet{bai2019provably} and \citet{zhang2021reinforcement} studied low-switching algorithms for tabular Markov Decision Processes (MDPs), while \citet{wang2021provably, he2023nearly} and \citet{huang2022towards} extended the study to linear function approximation. The most relevant works to ours consider low-switching RL with general function approximation. For example, \citet{kong2021online} proposed a low-switching RL approach for episodic MDPs using an online subsampling technique, \citet{zhao2023nearly} explored low-switching RL through the lens of a generalized Eluder dimension, and \citet{xiong2023general} studied a low-switching RL framework under a general Eluder condition class. Our work differs from these prior studies in two key aspects. From an algorithmic perspective, we develop a CTRL-based approach, which contrasts with existing methods designed for discrete episodic RL. From a theoretical standpoint, we introduce novel analytical tools and notions to handle the continuous-time nature of our dynamics.

\section{Preliminaries}

\noindent\textbf{Diffusion SDE}
In this work, we consider a general nonlinear continuous time dynamical system with a state $x(t)\in\mathcal{X}\subseteq  \mathbb{R}^d$ and a control unit $u(t)=\pi(x(t))\in\mathcal{U}\subseteq  \mathbb{R}^m, t\in [T]$. We model the system dynamics using an Itô-form stochastic differential equation (SDE), which is a tuple $\cS = \{f^*, g^*, b^*\}$. Given an initial distribution $q\in \cQ: \Delta(\cX)$, let the initial state $x(0) \sim q$, then the flow $x(t), t \in [0,T]$ is evolved following: 
\begin{small}
    \begin{align}
    dx(t)=f^*(x(t),u(t))dt+g^*(x(t),u(t)) dw(t),\label{def:dynamic}
\end{align}
\end{small}
where $f^* \in \cF: \cX\times \cU\rightarrow \mathbb{R}^d$ is the drift term and $g^*: \cX\times \cU \to \mathbb{R}$ is the diffusion term, and $w(t) \in \mathbb{R}^d$ is a standard Wiener process. Our goal is to find a deterministic policy $\pi\in \Pi: \cX\rightarrow \cU$ and an initial distribution $q \in \cQ: \Delta(\cX)$ which maximizes the following quantity:
\begin{align}
    R(\pi, q):=\EE\bigg[\int_{t=0}^T b^*(x(t), \pi(x(t)))dt\bigg|x(0)\sim q\bigg],\notag
\end{align}
where $b^* \in \cR: \cX\times \cU \to [0,1]$ denotes the reward function. 
Note that we only assume $f^*, b^*$ are unknown, and we can access $g^*$ during the algorithm. We only consider time-homogeneous policy. For time-inhomogeneous policy, we augment the state $x'(t) = [x(t), t]$. 
\begin{remark}\label{remark: sde}
    Our formulation of continuous time dynamical system in \eqref{def:dynamic} is general enough to capture many popular applications w.r.t. CTRL. A concrete example is given by diffusion models \citep{song2020score}, where one could formulate the backward process as follows:
    \begin{align}
        \label{def:diffusion}
        dx(t) \;=\; f(t,x(t))\,dt \;+\; \sigma(t)dw(t),
        \quad x(0) \sim q,
    \end{align}
    where $f$ is the standard drift formulated by neueral networks, and $\sigma(t)$ is a predefined diffusion term. Notably $f$ can be trained by either score matching \citep{song2020denoising, ho2020denoising} or flow matching \citep{lipman2022flow, shi2024diffusion, tong2023improving, somnath2023aligned, albergo2023stochastic, liu20232, liu2022let}. 
    Comparing \eqref{def:dynamic} and \eqref{def:diffusion}, it is straightforeward to see \eqref{def:diffusion} falls into the definition of our \eqref{def:ct}. 
\end{remark}

For simplicity, we use $X(t,\pi, q)$ to denote the random variable $x(t)$ following policy $\pi$ and the initial distribution $q$. We also denote $z = (x,u)$, and we use $Z(t,\pi,q)$ to denote the random variable $(x(t), \pi(x(t)))$ following policy $\pi$ and the initial distribution $q$.

\noindent\textbf{Measurement Model}
Everytime for a policy $\pi$ and the initial distribution $q$, we can choose a time $t$ to observe the following $(x(t),u(t),y(t),r(t))$, where
\begin{align}
    &x(t) = X(t, \pi, q), u(t) = \pi(x(t)),\notag\\
    &y(t)\sim  \cN(f^*(x,u), \frac{g^*(x,u)^2}{\Delta}\cdot I), r(t)\sim \cN(b^*(x,u),1),\notag
\end{align}
where $\cN(\mu, \sigma^2)$ denotes the normal distribution and $\Delta>0$ denotes the measurement time step.  
\begin{remark}
    We assume known diffusion coefficient $g^*(x,u)$ to simplify the theoretical analysis. This assumption is common in related literature — for instance, in diffusion-model-based RL fine-tuning \citep{uehara2024feedback, song2020denoising}, where $g^* = \sigma$ as discussed in Remark \ref{remark: sde}, and in deterministic dynamic control problems \citep{yildiz2021continuous}, where $g^* = 0$.
\end{remark}
\begin{remark}\label{remark:measurement}
In practice, only the state $x(t)$, control $u(t)$, and reward $r(t)$ are directly observed, whereas the instantaneous drift $y(t)$ must be approximated. Following the gradient-measurement approach in \citet{treven2024efficient} (Definition 1), we assume that both $x(t)$ and $x(t+\Delta)$ can be accessed jointly—effectively doubling the state observation cost, which is often reasonable when dense trajectory data are available. Then, for a sufficiently small time step $\Delta$, we apply the Euler–Maruyama method \citep{platen2010numerical} to approximate
$$
y(t)\;\approx\;\frac{x(t+\Delta)-x(t)}{\Delta}.
$$
With $\Delta$ small enough, this yields a valid approximation of the true instantaneous drift $y(t)$.
\end{remark}
In the following sections, we often omit $t$ and directly use $(x,u,y,r)$ to describe any measurements we will receive during a policy execution.

\noindent\textbf{Distributional Eluder Dimension}  
We introduce the notion of the $\ell_p$-distributional Eluder dimension \citep{jin2021bellman}, which extends the classical Eluder dimension \citep{russo2013eluder} to a distributional setting. Given a domain $\mathcal{A}$, a function class $\mathcal{B} \subseteq \mathcal{A} \to \mathbb{R}$, a distribution class $\mathcal{C} \subseteq \Delta(\mathcal{A})$, and a threshold parameter $\epsilon > 0$, we define the $\ell_p$-distributional Eluder dimension as $\text{DE}_p(\mathcal{A}, \mathcal{B}, \mathcal{C}, \epsilon)$, which is the largest integer $L$ such that there exists a sequence of distributions $p_1, \dots, p_L \subseteq \mathcal{C}$ satisfying the following condition: there exists a threshold $\epsilon' \geq \epsilon$ such that for all $l \in [L]$, there exists a function $h \in \mathcal{B}$ for which  
\[
\left| \mathbb{E}_{p_l} h \right| > \epsilon \quad \text{and} \quad \sum_{i=1}^{l-1} \left| \mathbb{E}_{p_i} h \right|^p \leq \epsilon'^p.
\]  

Intuitively, the distributional Eluder dimension quantifies the complexity of function class $\mathcal{B}, \mathcal{C}$ by capturing the nonlinearity of the expectation operator $\mathbb{E}_{p_l} h$. In this work, we leverage this measure as the key complexity metric to characterize the nonlinearity in our continuous-time dynamical system.

\section{Provable CTRL with General Function Approximation}
\begin{algorithm*}
\caption{$\algbase$}
\label{alg:alg1}
\begin{algorithmic}[1]
\REQUIRE Total measurement number $N$, policy class $\Pi$, initial distribution class $\cQ$, drift class $\cF$, diffusion term $g^*$, reward class $\cR$, confidence radius $\beta_{\cF}$, $\beta_{\cR}$, episode length $T$, initial measurement set $\cD_1 = \emptyset$
\STATE Initialize confidence sets $\cF_1 = \cF, \cR_1 = \cR$. 
\FOR{episode $n = 1,\dots, N$}
\STATE Set $\pi_n, q_n, f_n, b_n \leftarrow \argmax_{\pi \in \Pi, q \in \cQ,f \in \cF_n, b \in \cR_n} R(\pi,q, f,b)$. 
\STATE Uniformly sample $t_n \in \text{Unif}[0, T]$. Execute $q_n,\pi_n$ and receives measurement $(x_n, u_n, y_n, r_n)$ at time $t_n$. Update $\cD_{n+1} \leftarrow \cD_n\cup (x_n, u_n, y_n, r_n)$. 
\STATE Update $\cF_{n+1}$ following \eqref{alg:fn}, update $\cR_{n+1}$ following \eqref{alg:rn}. 
\ENDFOR
\ENSURE Randomly pick up $n\in[N]$ uniformly and outputs $(\hat \pi, \hat q)$ as $(\pi_n, q_n)$. 
\end{algorithmic}
\end{algorithm*}

In this section, we introduce $\algbase$, outlined in Algorithm \ref{alg:alg1}. Broadly speaking, $\algbase$ is a model-based CTRL algorithm that interacts with the environment online, receives feedback, and continuously updates its estimates of the dynamics $f^*$ and reward function $b^*$. During the $n$-th episode, $\algbase$ maintains confidence sets for $f^*$ and $b^*$, denoted as $\cF_n$ and $\cR_n$, respectively. Formally, given a dataset $\cD = \{(x,u,y,r)\}$, we define the empirical loss functions for the dynamics and reward as  
\begin{align}
    &L_{\cD}(f) = \sum_{(x,u,y,r) \in \cD} (f(x,u) - y)^2, \notag \\
    &L_{\cD}(b) = \sum_{(x,u,y,r) \in \cD} (b(x,u) - r)^2. \notag
\end{align}
Let $\cD_n$ be the collection of all measurements $(x,u,y,r)$ collected up to episode $(n-1)$. The confidence sets $\cF_n$ and $\cR_n$ are then constructed as follows:  
\begin{align}
    &\cF_{n} \leftarrow \big\{ f \mid L_{\cD_{n}}(f) \leq \min_{f' \in \cF} L_{\cD_{n}}(f') + \beta_{\cF} \big\}, \label{alg:fn}\\
    &\cR_{n} \leftarrow \big\{ b \mid L_{\cD_{n}}(b) \leq \min_{b' \in \cR} L_{\cD_{n}}(b') + \beta_{\cR} \big\}. \label{alg:rn}
\end{align}

Following the classical optimism-in-the-face-of-uncertainty principle \citep{abbasi2011improved}, $\algbase$ jointly optimizes its policy, initial distribution, and estimates of the dynamics and reward functions to maximize the accumulated reward $R(\pi, q, f, b)$. At each episode, it uniformly samples a time step $t_n \in [T]$, executes the policy $\pi_n$ under the initial state distribution $q_n$, and receives the measurement $(x_n, u_n, y_n, r_n)$ at time $t_n$. After running for $N$ episodes, $\algbase$ outputs the target policy and initial distribution by selecting uniformly at random from the existing ones.

\begin{remark}
    We briefly compare $\algbase$ with OCoRL \citep{treven2024efficient}, which is the most closely related algorithm. A key difference is that OCoRL requires access to an external oracle that quantifies epistemic uncertainty in estimating $f^*$, whereas $\algbase$ operates without explicitly maintaining such an oracle. Additionally, OCoRL selects $t_n$ deterministically based on complex strategies and assumes additional smoothness conditions on the epistemic uncertainty oracle. In contrast, $\algbase$ employs a simple and randomized selection of $t_n$, making it more flexible and potentially more practical in real-world applications.
\end{remark}

Next we provide theoretical analysis for $\algbase$. 
We first make the following assumptions on $f,b,g,\pi$. 
\begin{assumption}\label{Lipshitzness}
    We have that for any $\pi \in \Pi$, any $f\in \cF$, $b \in \cR$, any $x,x' \in \cX$, $u,u' \in \mathcal{U}$, 
    \begin{itemize}[leftmargin = *]
        \item We have the following bounded assumptions: $|f(x,u)| \leq 1$, $|b(x,u)| \leq 1$ and $|g(x,u)| \leq G/\sqrt{d\Delta}$. 
        \item We have the following Lipschitz-continuous assumptions:
            \begin{align}
        &\|f(x, u) - f(x', u')\|_2 \leq L_f(\|x - x'\|_2+\|u - u'\|_2),\notag \\ 
        &|b(x, u) - b(x', u')| \leq L_b(\|x - x'\|_2+\|u - u'\|_2),\notag \\ 
        &|g(x, u) - g(x', u')| \leq L_g(\|x - x'\|_2+\|u - u'\|_2),\notag\\
        &\|\pi(x) - \pi(x')\|_2 \leq L_\pi\|x - x'\|_2.\notag
    \end{align}
    \end{itemize}

\end{assumption}

Next we define the distributional Eluder dimension of the function class $\algbase$ used in its algorithm design. 
\begin{definition}
Let $\mathcal{Z} = \cX \times \cU$ and $p \in \cP$ denote the following distribution class over $\mathcal{Z}$: each distribution $p$ is associated with a policy $\pi \in \Pi$ and an initial distribution $q \in \cQ$, such that 
\[
z = (x, u) \sim p :
\left\{
\begin{aligned}
x &= X(t, \pi, q), \, t \sim \mathrm{Unif}[0, T], \\
u &= \pi(x).
\end{aligned}
\right.
\]
Furthermore, denote function class $\bar\cF, \bar\cR$ as 
\begin{align}
    \bar\cF= \{\|f - f^*\|_2^2: f \in \cF\},\bar\cR = \{(b - b^*)^2: b \in \cR\}.\notag
\end{align}
Then we set 
\begin{align}
    d_{\cF,\epsilon} = \text{DE}_1(\cZ, \bar\cF, \cP, \epsilon),\ d_{\cR,\epsilon} = \text{DE}_1(\cZ, \bar\cR, \cP, \epsilon).\notag
\end{align}
\end{definition}
Intuitively, $d_{\cF,\epsilon}$ and $d_{\cR,\epsilon}$ capture the nonlinearity of loss functions $\|f - f^*\|_2^2$ and $(b - b^*)^2$ with respect to the dynamical system induced by function class $\cF$ and policy class $\Pi$. Next proposition shows that several common models, including linear dynamical systems, enjoy small distributional Eluder dimensions. The proof is deferred to Appendix \ref{proof:prop1}.

\begin{proposition}\label{smallelu}
Let $\cF = \{ f_{\theta}(z) = \langle \Theta, \phi(z) \rangle : \|\Theta\|_F \leq R \}$ and $\cR = \{ b_\theta(z) = \langle \theta, \phi(z) \rangle : \|\theta\| \leq R \}$, where $\Theta \in \RR^{d \times d}$ and $\theta \in \RR^d$. These represent classes of linear functions on $\cZ$ with a feature mapping $\phi$. Assume that $\|\phi(z)\| \leq L$ for all $z \in \mathcal{Z}$. Then, for any family $\cP$ of distributions on $\mathcal{Z}$ and for any $\epsilon > 0$, we have $d_{\cF,\epsilon} = O(d^2\log(1+R^2L^2/\epsilon^2)), d_{\cR,\epsilon} = O(d^2\log(1+R^2L^2/\epsilon^2))$. 
\end{proposition}
Next we show our main algorithm which characterizes the sample complexity of $\algbase$. 
\begin{theorem}\label{thm:1}
    Set confidence radius $\beta_{\cF}, \beta_{\cR}$ as
    \begin{align}
        &\beta_{\cF} = O(G^2\log(N\cdot\log N|\cN(\cF, N^{-2})|/\delta) ),\notag \\
        & \beta_{\cR} = O(\log(N\cdot\log N|\cN(\cR, N^{-2})|/\delta) ),\notag
    \end{align}
    then with probability at least $1-\delta$, Algorithm \ref{alg:alg1} satisfies:
    \begin{itemize}[leftmargin = *]
        \item For all $n \in [N]$, $f^* \in \cF_n, b^* \in \cR_n$. 
        \item The suboptimality gap of $(\hat \pi, \hat q)$ is bounded by
        \begin{small}
      \begin{align}
O\bigg(\frac{T\sqrt{d_{\cR, N^{-1}}\beta_{\cR}}+LT^{3/2}\sqrt{\exp(K T)}\sqrt{d_{\cF, N^{-1}}\beta_{\cF}}}{\sqrt{N/\log N}}\bigg),\label{help:sample}
    \end{align}          
        \end{small}
    where $K=1 + (1+L_\pi)^2\cdot L_g^2 + 2(1+L_\pi)^2\cdot L_f^2,L=L_b(1+L_\pi)$.
    \end{itemize}
\end{theorem}

From Theorem \ref{thm:1}, we know that to find an $\epsilon$-optimal policy $\hat \pi$, it is sufficient to set the total measurement number 
\begin{align}
    N &= \tilde O\bigg(\epsilon^{-2}\cdot\big(T^2d_{\cR, \epsilon^{-2}}\log(|\cN(\cR, \epsilon^{-4})|/\delta)\notag \\
    \quad &+L^2T^3\exp(K T)d_{\cF, \epsilon^{-2}}G^2\log(|\cN(\cF, \epsilon^{-4})|/\delta)\big)\bigg),\notag
\end{align}
which gives us an $\tilde O(\epsilon^{-2})$ sample complexity if we treat the Eluder dimensions and covering numbers as constants.

\begin{remark}
    Note that our sample complexity has an exponential dependence on the time horizon $T$, which seems larger than the polynomial dependence of the planning horizon in discrete-time RLs \citep{jin2021bellman}. However, we would like to highlight that CTRL and discrete-time RL are two different types of algorithms for different problem settings, thus they can not be compared directly. Meanwhile, our result also aligns with several recent works about CTRL \citep{treven2024efficient}, which also established exponential dependence on $T$. 
\end{remark}

\begin{remark}
    We assume a perfect approximation of $R$ to simplify the theoretical analysis. However, it is not difficult to extend our results to account for a $\ell$-approximation. Suppose that in each episode $n$, we can only access an estimate $\hat R_n$ such that $|\hat R_n - R| \leq \ell$ for some $\ell > 0$. Then, it is not difficult to show that the suboptimality gap will be bounded by \eqref{help:sample} with an additional $\ell$ factor. This ensures that the approximation error does not significantly impact the overall performance of the algorithm.
\end{remark}

\subsection{Proof Sketch} Below is the proof sketch of Theorem \ref{thm:1}, with full proof deferred to Appendix \ref{proof:thm:1}. We mainly demonstrate how to bound the regret $\sum_{n=1}^N R_n$, where
\begin{align}
    R_n &= R(f^*,b^*,\pi^*,q^*)-R(f^*,b^*,\pi_n,q_n).\notag
\end{align}
\begin{enumerate}[leftmargin = *]
  \item \textbf{Trajectory deviation.}
    Applying Itô’s lemma, Fubini’s theorem, Grönwall’s inequality and standard analytic arguments, the mean–square gap between the true trajectory $x_n(t)$ and the optimistic trajectory $\hat x_n(t)$ is 
    \begin{align} \small
        &\EE\bigl\|\hat x_n(t)-x_n(t)\bigr\|^{2}\notag \\
      &\le
      2e^{Kt}\int_{0}^{t}\EE\bigl\|f^{*}(x_n(s),\pi_n(x_n(s))) \notag\\ & \quad -f_n(x_n(s),\pi_n(x_n(s)))\bigr\|^{2}ds. \notag
    \end{align}

  \item \textbf{High-probability confidence sets.} Standard covering argument yields empirical–risk inequalities for any $b\in\cR$ and $f\in\cF$.  Choosing confidence radii $\beta_{\cR},\beta_{\cF}=\tilde O(\log(|\cdot|/\delta))$ guarantees $b^{*}\in\cR_{n}$ and $f^{*}\in\cF_{n}$ for all $n$ with probability $1-\delta$.

  \item \textbf{Per-episode regret decomposition.}
    By optimism,
    \begin{align}\small
        R_n \le R(f_n,b_n,\pi_n,q_n)-R(f^*,b^*,\pi_n,q_n). \notag
    \end{align}
    Using Lipschitz continuity and Cauchy–Schwarz inequality, one shows
    \begin{align}\small
    R_n &\le L\,\EE\int_{0}^{T}\|\hat x_n(t)-x_n(t)\|dt \notag \\
    & \quad + \EE\int_{0}^{T}(b_n-b^*)(x_n(t),\pi_n(x_n(t)))\,dt,\notag
    \end{align}
    which, together with the trajectory bound, gives
    $$
      R_n \;\le\;LT\sqrt{2Te^{KT}A_n}\;+\;T\sqrt{B_n},
    $$
    where $A_n=\EE\|f_n-f^*\|^{2}$, $B_n=\EE|b_n-b^*|^{2}$.

  \item \textbf{Chaining via Eluder dimension.}
    Applying Theorem 5.3 from \citet{wang2023benefits} to the sequences $\{(b_n-b^*)^2\}$ and $\{\|f_n-f^*\|^2\}$ converts the confidence radii $\beta_{\cR},\beta_{\cF}$ into
    $\sum_{n}B_n=O(d_{\cR}\beta_{\cR}\log N)$ and
    $\sum_{n}A_n=O(d_{\cF}\beta_{\cF}\log N)$.

  \item \textbf{Cumulative regret.}
Summing the regret decomposition from step 3 over $n=1,\dots,N$ and applying the Cauchy--Schwarz inequality yield, with probability at least $1-2\delta\log N$,
\begin{align}\small
\sum_{n=1}^{N}R_n &=O\!\Bigl( T\sqrt{N\,d_{\cR}(\log N+\log|\cR_{\epsilon}|)} \notag \\ & \quad +LT\sqrt{TN\,e^{KT}\,d_{\cF}(\log N+\log|\cF_{\epsilon}|)} \Bigr).\notag
\end{align}
    Replacing $\delta$ by $\delta/(2\log N)$ in the confidence parameter leaves the asymptotic rate unchanged, and thus the theorem follows.
\end{enumerate}

\section{Improved Computational Efficiency for $\algname$}

$\algbase$ suggests that to find an $\epsilon$-optimal policy, $\tilde{O}(\epsilon^{-2})$ measurements are required. This dependency aligns with the standard statistical error rate established in prior works. However, a key limitation of $\algbase$ is that it updates its policy and initial distribution in every episode, which can be computationally expensive if such updates are costly. Additionally, $\algbase$ collects only a single uniformly random measurement per episode. While this ensures sample efficiency, it also results in wasted rollouts, as the policy is executed at all times but only measured at one. In contrast, discrete-time RL evaluates the policy at every time step, making it more computationally efficient. To address these two challenges, we propose two improved versions of $\algbase$, each designed to tackle a specific limitation.

\subsection{Policy update efficient $\algname$}

\begin{algorithm*}
\caption{$\algpolicy$}
\label{alg:alg2}
\begin{algorithmic}[1]
\REQUIRE Total measurement number $N$, policy class $\Pi$, initial distribution class $\cQ$, drift class $\cF$, diffusion term $g^*$, reward class $\cR$, confidence radius $\beta_{\cF}$, $\beta_{\cR}$, episode length $T$, initial measurement set $\cD_1 = \emptyset$.
\STATE Initialize confidence sets $\cF_1 = \cF, \cR_1 = \cR$. 
\FOR{episode $n = 1,\dots, N$}
\STATE Set $\pi_n, q_n, f_n, b_n \leftarrow \argmax_{\pi \in \Pi, q \in \cQ,f \in \cF_n, b \in \cR_n} R(\pi,q, f,b)$. 
\STATE Uniformly sample $t_n \in \text{Unif}[0, T]$. Execute $q_n,\pi_n$ and receives measurement $(x_n, u_n, y_n, r_n)$ at time $t_n$. Update $\cD_{n+1} \leftarrow \cD_n\cup (x_n, u_n, y_n, r_n)$. 
\STATE Set $\cF_{n+1} \leftarrow \cF_{n}$, $\cR_{n+1} \leftarrow \cR_{n}$
\STATE \textbf{if}\ { $L_{\cD_{n+1}}(f_n) \geq \min_{f'\in \cF}L_{\cD_{n+1}}(f') + 5\beta_{\cF}$} \textbf{then} Update $\cF_{n+1}$ following \eqref{alg:fn}
\STATE \textbf{if}\ {$L_{\cD_{n+1}}(b_n) \geq \min_{b'\in \cR}L_{\cD_{n+1}}(b') + 5\beta_{\cR}$} \textbf{then} Update $\cR_{n+1}$ following \eqref{alg:rn}. 
\ENDFOR
\ENSURE Randomly pick up $n\in[N]$ uniformly and outputs $(\hat \pi, \hat q)$ as $(\pi_n, q_n)$. 
\end{algorithmic}
\end{algorithm*}

We first introduce $\algpolicy$ in Algorithm \ref{alg:alg2}, designed to reduce the frequency of policy and initial distribution updates. In essence, $\algpolicy$ follows the same setup as $\algbase$ while actively monitoring how well the estimated model fits the collected measurements. Specifically, it updates the dynamic model $f_n$ and reward model $b_n$ only when either fails to fit the current dataset $\cD_{n+1}$—that is, when the empirical loss $L_{\cD_{n+1}}(f_n)$ exceeds a predefined threshold. When such a discrepancy is detected, $\algpolicy$ updates the corresponding model and adjusts the policy and initial distribution accordingly.  

Next, we present the theoretical guarantees for $\algpolicy$.  

\begin{theorem}\label{thm:2}
    Setting the confidence radii $\beta_{\cF}, \beta_{\cR}$ as in Theorem \ref{thm:1}, with probability at least $1 - \delta$, Algorithm \ref{alg:alg2} satisfies:
    \begin{itemize}[leftmargin = *]
        \item For all $n \in [N]$, $f^* \in \cF_n$ and $b^* \in \cR_n$.
        \item The suboptimality gap of $(\hat \pi, \hat q)$ matches that in \eqref{help:sample}.
        \item The total number of episodes where $\pi_n$ and $q_n$ are updated is at most $\log N \cdot O(d_{\cF, N^{-1}} + d_{\cR, N^{-1}})$. 
    \end{itemize}
\end{theorem}

The proof is deferred to Appendix \ref{app:thm2}. 
The result above implies that $\algpolicy$ significantly reduces the number of policy and initial distribution updates from $N$ to $\log N$, without degrading the final policy performance.  

\begin{remark}
    A similar strategy has been explored in prior works on discrete-time RL with general function approximation \citep{xiong2023general, zhao2023nearly}. However, in discrete-time RL, these methods monitor \emph{every} discrete time step $t$ to detect discrepancies in the estimated dynamics. In contrast, such an approach is infeasible in CTRL, as continuous-time monitoring is not possible. This fundamental difference makes our analysis of $\algpolicy$ significantly more challenging.
\end{remark}

\subsection{Rollout efficient $\algname$}

\begin{algorithm*}
\caption{$\algmss$}
\label{alg:alg3}
\begin{algorithmic}[1]
\REQUIRE Total measurement number $N$, policy class $\Pi$, initial distribution class $\cQ$, drift class $\cF$, diffusion term $g^*$, reward class $\cR$, confidence radius $\beta_{\cF}$, $\beta_{\cR}$, episode length $T$, measurement frequency $m$, sampler $\cT$, initial measurement set $\cD_1 = \emptyset$. 
\STATE Initialize confidence sets $\cF_1 = \cF, \cR_1 = \cR$. 
\FOR{episode $n = 1,\dots, N/m$}
\STATE Set $\pi_n, q_n, f_n, b_n \leftarrow \argmax_{\pi \in \Pi, q \in \cQ, f \in \cF_n, b \in \cR_n} R(\pi,q,f,b)$. 
\STATE Sample $t_{n,1},\dots, t_{n,m}\sim \cT$. Execute $q_n, \pi_n$ and receive measurement $(x_{n,i}, u_{n,i}, y_{n,i}, r_{n,i})$ at time $t_{n,i}$. Update $\cD_{n+1} \leftarrow \cD_n\cup \{(x_{n,i}, u_{n,i}, y_{n,i}, r_{n,i})\}_{i=1}^{m}$. 
\STATE Update $\cF_{n+1}$ following \eqref{alg:fn}, update $\cR_{n+1}$ following \eqref{alg:rn}. 
\ENDFOR
\ENSURE Randomly pick up $n\in[N/m]$ uniformly and outputs $(\hat \pi, \hat q)$ as $(\pi_n, q_n)$. 
\end{algorithmic}
\end{algorithm*}

Next, we study how to reduce the number of rollouts required by $\algbase$. To achieve this, we propose $\algmss$, outlined in Algorithm \ref{alg:alg3}, which performs multiple measurements within a single episode. Specifically, in the $n$-th episode, measurements are taken at times $t_{n,1}, \dots, t_{n,m}$, where $m$ is the \emph{measurement frequency}. For simplicity, we analyze a fixed measurement strategy, assuming that for any $n \in [N/m]$, the measurement times $t_{n,1}, \dots, t_{n,m}$ are sampled from a predefined distribution $\cT$. Here, $\cT$ is allowed to be any joint distribution over $\text{Unif}[0,T]^m$. In each episode, the dataset $\cD_n$ is updated in batches, incorporating $m$ measurements $\{(x_{n,i}, u_{n,i}, y_{n,i}, r_{n,i})\}_{i=1}^{m}$ collected during the episode. Other than this batched measurement update, $\algmss$ follows the same procedure as $\algbase$.  

By introducing the measurement frequency $m$ and the sampler $\cT$, $\algmss$ reduces the number of policy rollouts from $N$ to $N/m$. To maintain a fair comparison and ensure consistency in sample complexity, we also set the total number of episodes to be $N/m$, ensuring that $\algmss$ and $\algbase$ use the same total number of measurements, differing only in the number of rollouts.  

\begin{remark}
    Our in-episode sampling strategy is similar to the Measurement-Selection-Strategy (MSS) introduced in \citet{treven2024efficient}. However, a key distinction is that we do not impose determinism or any specific structure on the sampler $\cT$, allowing for a more general and flexible measurement selection process.
\end{remark}

Next, we analyze how the introduced measurement frequency affects the output policy and initial distribution. To quantify this effect, we define the \emph{independency coefficient} of $\cT$, which measures how well our sampler approximates i.i.d. samples.  

\begin{definition}\label{def:ct}
Given a policy $\pi$, an initial distribution $q$, and a sampling strategy $\cT$, let $\hat Z$ be the random variable defined as $\hat Z = Z(t, \pi, q)$, where $t \sim \text{Unif}[0,T]$. Let $\bar Z_1, \dots, \bar Z_m$ be random variables corresponding to measurement times $t_1, \dots, t_m \sim \cT$, where $\bar Z(t)$ denotes a trajectory sampled according to $\pi, q$, and $\bar Z_i = \bar Z(t_i)$. We define the \emph{independency coefficient} $C_{\cT,m}$ as $C_{\cT,m}:=\sup_{i \in [m]}\max\{C_{\cT,m, \cF,i}, C_{\cT,m, \cR,i}\}$, where
\begin{small}
    \begin{align}
    C_{\cT,m, \cF,i}:=\sup_{\bar Z_{i-1}\dots, \bar Z_{1}, \pi,q} \frac{\EE_{z_i\sim\mathbb{P}_{\hat Z}}\|f(z_i) - f^*(z_i)\|_2^2}{\EE_{z_i\sim\mathbb{P}_{\bar Z_i \mid \bar Z_{i-1}, \dots, \bar Z_1}}\|f(z_i) - f^*(z_i)\|_2^2},\notag \\
    C_{\cT,m, \cR,i}:=\sup_{\bar Z_{i-1}\dots, \bar Z_{1}, \pi,q} \frac{\EE_{z_i\sim\mathbb{P}_{\hat Z}}(b(z_i) - b^*(z_i))^2}{\EE_{z_i\sim\mathbb{P}_{\bar Z_i \mid \bar Z_{i-1}, \dots, \bar Z_1}}(b(z_i) - b^*(z_i))^2},\notag
    \end{align}
\end{small}
\end{definition}

Intuitively, $C_{\cT,m}$ quantifies how well the measurement times generated by $\cT$ approximate those obtained from uniform sampling per rollout. The following proposition suggests that for certain continuous-time dynamical systems, $C_{\cT,m}$ can be upper bounded by a small constant. The proof is deferred to Appendix \ref{proof:prop2}.

\begin{proposition}\label{prop:22}
    There exists a one-dimensional continuous-time dynamical system with the control space $\cU$ being one-dimensional and lower bounded by $u_{\min}>0$, satisfying $C_{\cT,m} \leq 1 + \frac{m}{2T u_{\min}}$. For this dynamical system, we have $C_{\cT,m}\leq 2$ when $m\leq 2Tu_{\min}$. 
\end{proposition}

Using $C_{\cT,m}$, we establish the following theoretical guarantee for $\algmss$.  

\begin{theorem}\label{thm:3}
    Setting the confidence radii $\beta_{\cF}, \beta_{\cR}$ as in Theorem \ref{thm:1}, with probability at least $1 - \delta$, Algorithm \ref{alg:alg3} satisfies:
    \begin{itemize}[leftmargin = *]
        \item For all $n \in [N/m]$, $f^* \in \cF_n$ and $b^* \in \cR_n$.  
        \item The suboptimality gap of $(\hat \pi, \hat q)$ is bounded by
        \begin{small}
        \begin{align}
        O\bigg(&\frac{T\sqrt{d_{\cR, N^{-1}}\beta_{\cR}} + L T^{3/2} \sqrt{\exp(KT)}\sqrt{d_{\cF, N^{-1}}\beta_{\cF}}}{\sqrt{N/\log N}} \notag \\ 
        &\quad \cdot \sqrt{C_{\cT,m}}
        + \frac{mT(d_{\cF, N^{-1}} + d_{\cR, N^{-1}})}{N/\log N} \bigg). \notag
        \end{align}
        \end{small}
        where $K=1 + (1+L_\pi)^2\cdot L_g^2 + 2(1+L_\pi)^2\cdot L_f^2,L=L_b(1+L_\pi)$.
    \end{itemize}
\end{theorem}

The proof is deferred to Appendix \ref{app:thm3}. 
By treating the Eluder dimension and covering numbers as constants, Theorem \ref{thm:3} suggests the following suboptimality gap:  
\begin{align}
    \tilde{O}\bigg(\frac{\sqrt{C_{\cT,m}}}{\sqrt{N}} + \frac{m}{N} \bigg).\label{help:444}
\end{align}

Comparing this result with the suboptimality gap in Theorem \ref{thm:1}, we draw the following conclusions. First, the quality of the final output policy and initial distribution depends on the effectiveness of the sampler $\cT$. The closer $\cT$ is to generating i.i.d. samples, the more similar the performance of $\algmss$ and $\algbase$. Second, if the sampler $\cT$ is sufficiently well-designed such that $C_{\cT,m} = O(1)$, then by \eqref{help:444}, we can safely increase $m$ without significantly compromising the final policy performance.

\section{Experiments}
In this section, we apply the principles of $\algbase$, $\algpolicy$, and $\algmss$ to several practical CTRL-based setups to evaluate their effectiveness. Specifically, we aim to answer the following question:  
\begin{center}
    \textit{Given the same number of measurements, can we reduce the total training time of CTRL by minimizing the number of policy updates and rollouts while maintaining comparable final performance to the original base algorithm?}
\end{center}  
To investigate this, we conduct experiments across two distinct domains: (1) fine-tuning diffusion models and (2) classical continuous control tasks.

\subsection{Fine-Tuning Diffusion Models} \label{exp:PURE-SEIKO}

\begin{figure*}[!htb]
\centering
\subfloat[\label{fig:seiko-base} $\algseiko$ vs. \textsc{SEIKO}]{\includegraphics[width=0.3\textwidth]{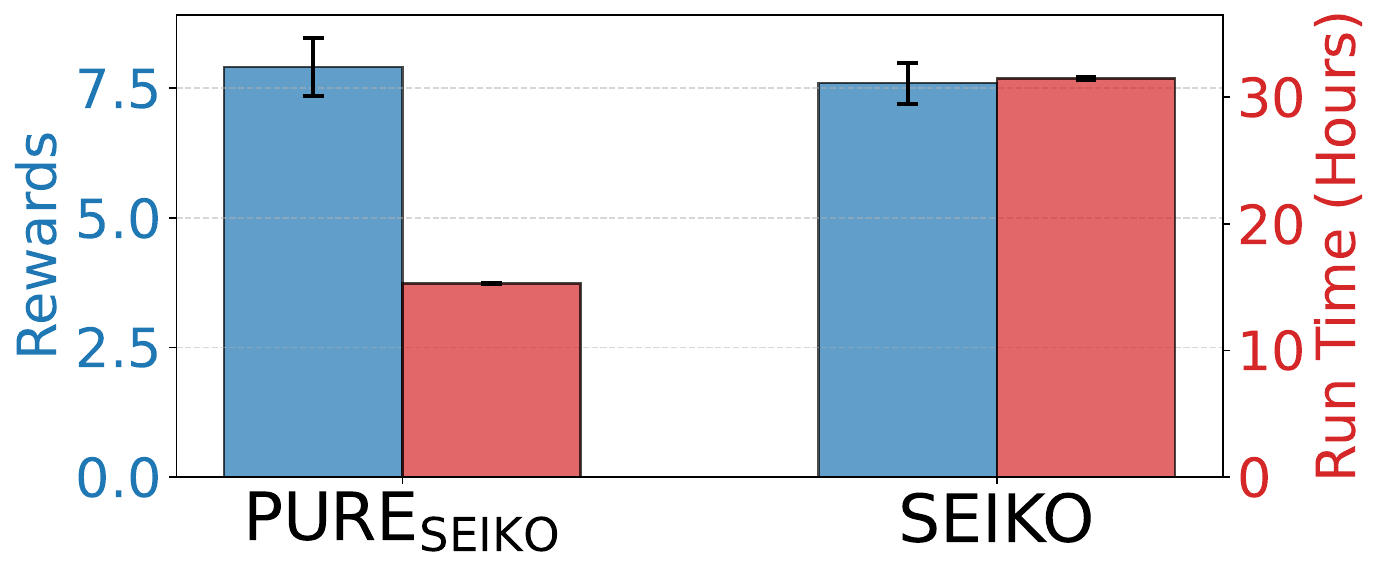}}\hfill
\subfloat[\label{fig:seiko-ablation1} $\algseiko$ with varying $\mathcal K$] {\includegraphics[width=0.3\textwidth]{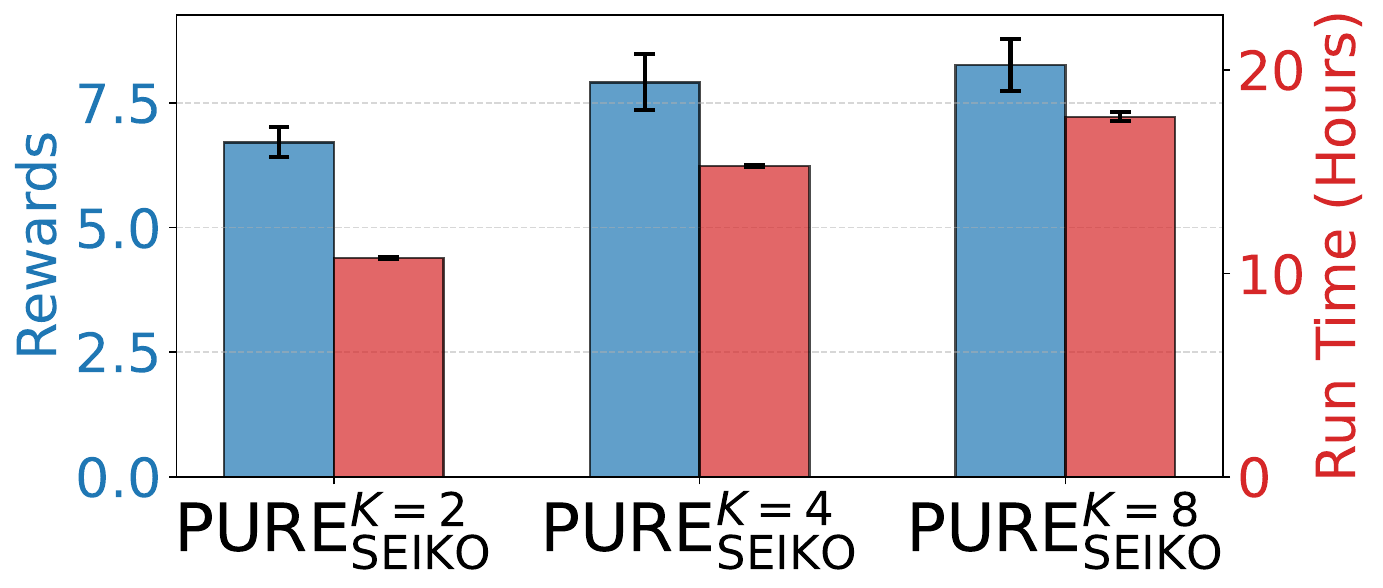}}\hfill
\subfloat[\label{fig:seiko-ablation2} $\algseiko$ with varying $m$]{\includegraphics[width=0.3\textwidth]{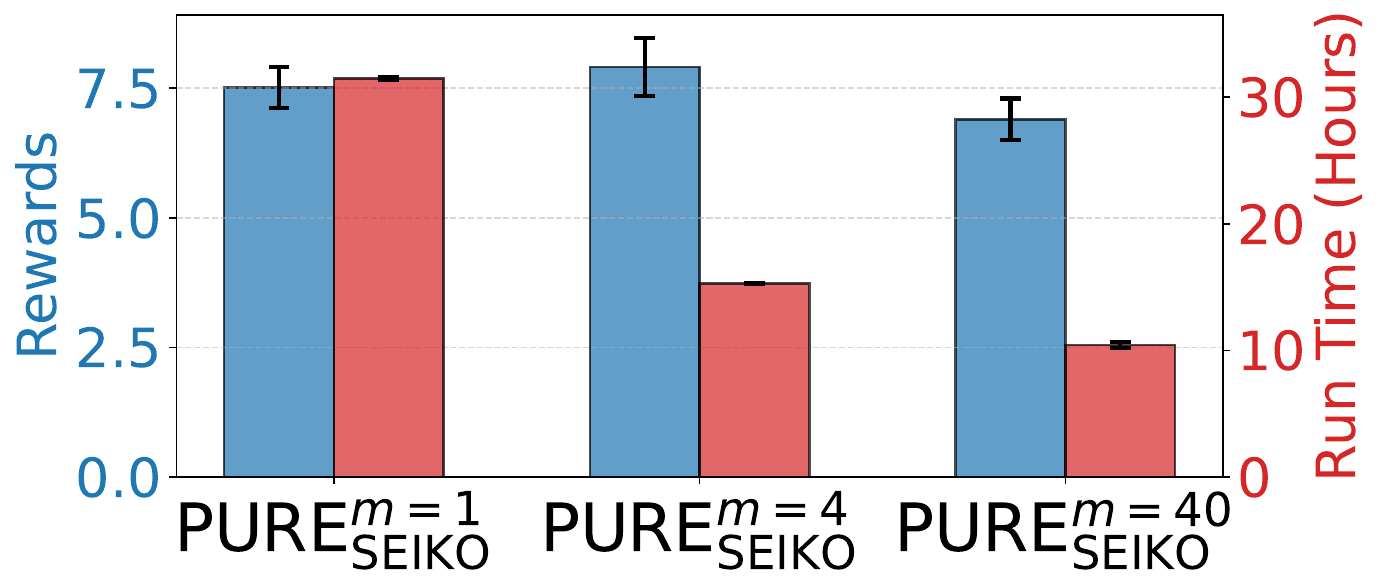}}
\caption{Summary of the experiment for fine-tuning Diffusion Models. \ref{fig:seiko-base} presents a comparison of aesthetic scores for denoised images generated by the fine-tuned Diffusion policy. \ref{fig:seiko-ablation1} and \ref{fig:seiko-ablation2} show ablation studies examining the effects of the number of policy updates and the value of $m$ on the final reward. 
} \label{fig:seiko}
\vspace{-0.5cm}
\end{figure*}

\begin{figure*}[htbp]
\centering
\subfloat[\label{fig:oderl-acrobot} Comparison on  Acrobot]{\includegraphics[width=0.3\textwidth]{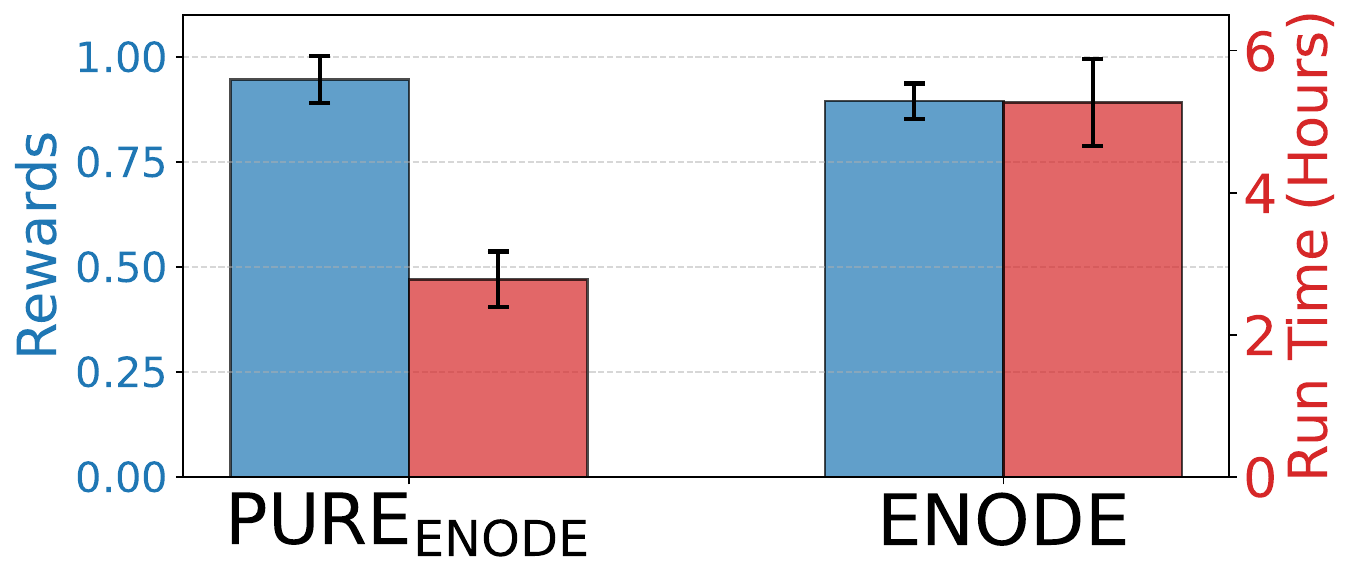}}\hfill
\subfloat[\label{fig:oderl-pendulum} Comparison on  Pendulum]{\includegraphics[width=0.3\textwidth]{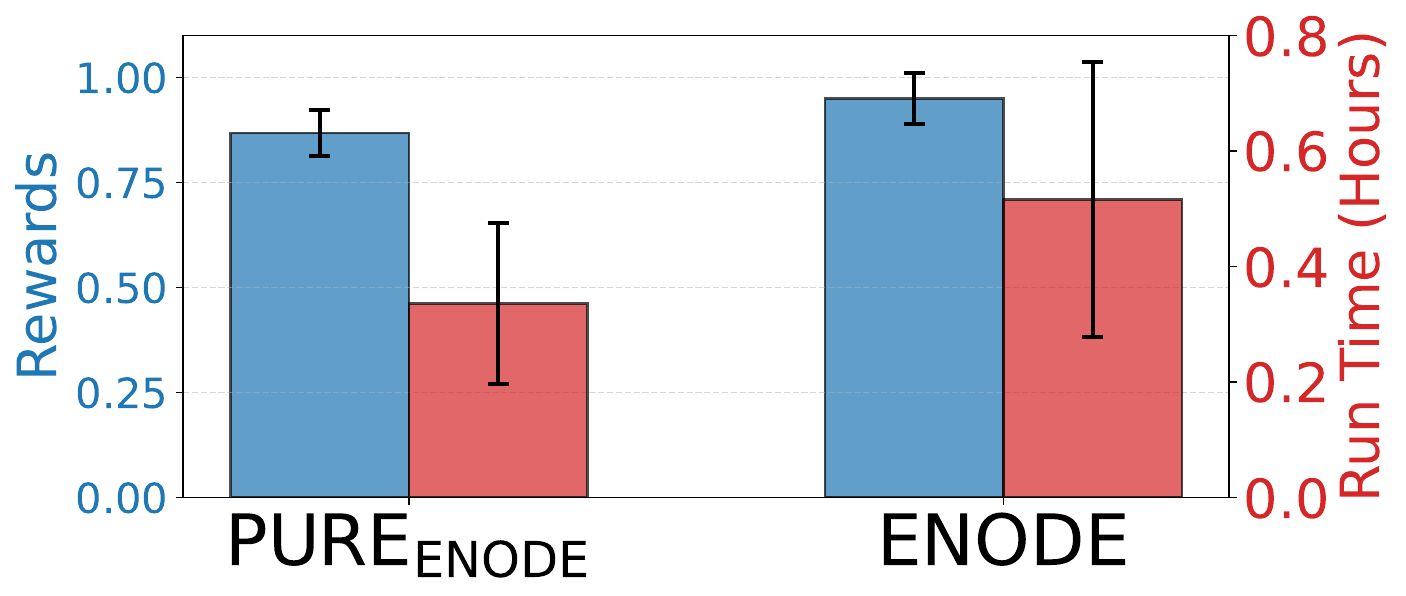}}\hfill
\subfloat[\label{fig:oderl-cartpole} Comparison on Cart Pole] {\includegraphics[width=0.3\textwidth]{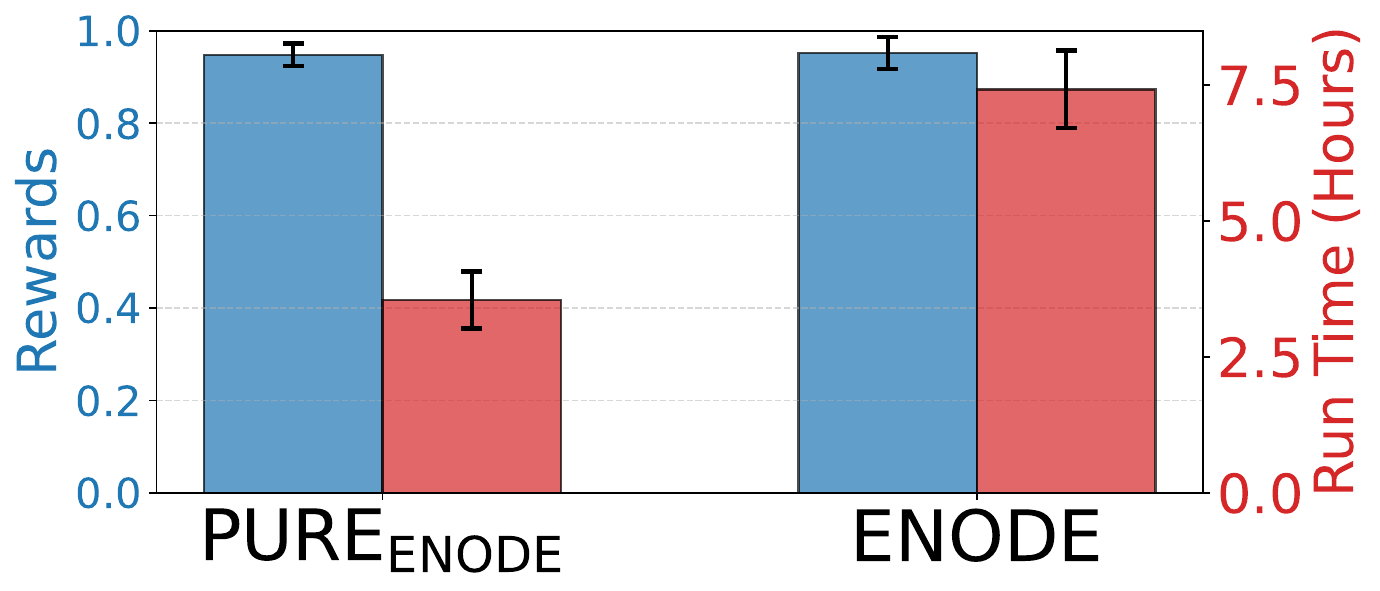}}
\caption{Summary of continuous-time control experiments, $\algenode$ vs. \textsc{ENODE}, in three control environments. } \label{fig:oderl}
\end{figure*}

\paragraph{Experiment Setup}

We consider fine-tuning a diffusion model to generate images with enhanced aesthetic quality, as measured by aesthetic scores \citep{rw2019timm, Radford2021LearningTV, liu2022convnet, schuhmann2022laionb, black2023training}. Our baseline fine-tuning approach is \textsc{SEIKO} \citep{uehara2024feedback}, a continuous-time reinforcement learning framework specifically tailored for optimizing diffusion models. \textsc{SEIKO} jointly refines the diffusion policy $\hat\pi$ and the initial distribution $\hat q$ by drawing samples from a pretrained diffusion backbone and training a reward function $\hat r$ based on these samples, where the reward encodes the aesthetic score as its evaluation metric. This setup dovetails perfectly with our $\algname$ framework under the assumption of fixed, known dynamics $f$. The measurement we receive is simplified to $(x,u,r)$ since $f$ does not need to be estimated. A key feature of \textsc{SEIKO} is its low-switching strategy, where episodes are divided into $\mathcal K$ batches, and the batch size $B_i$ is increased geometrically according to $B_{i+1} = \eta_{\text{base}} B_i$ for $i \in [\mathcal K]$. Due to space constraints, we defer details on SEIKO’s backbone architecture and the training procedure for updating $\hat\pi, \hat q$, and $\hat r$ to Appendix \ref{app:seiko_exp}.

\paragraph{Algorithm Implementation} \label{exp:seiko-implement}

We propose our algorithm, $\algseiko$, which builds upon \textsc{SEIKO} by incorporating an additional sampler, $\cT$, from $\algmss$ to reduce the number of rollouts and thereby lower the computational complexity of \textsc{SEIKO}. The full pseudo-code for $\algseiko$ is provided in Algorithm~\ref{alg:pure-seiko}, with further details on \textsc{SEIKO} and our modifications available in Appendix~\ref{app:seiko_exp}. Our sampler $\cT$ operates based on a measurement frequency parameter $m$. At the $n$-th episode, it samples $\{t_{n,1},\dots, t_{n,m}\} \subseteq \{\frac{1}{m}T,\dots,\frac{m-1}{m}T, T\}$. Each $t_{n,i}$ is drawn independently, with the probability of selecting $\frac{i}{m}T$ following a geometric distribution $\mathbb{P}\left(\frac{i}{m}T\right) \propto \lambda^{i}$, 
where $\lambda > 0$ is a tunable temperature parameter. In our experiments, we set $\lambda = 6$. Intuitively, geometric weighting aligns well with the exponential nature of information gain in reverse diffusion \citep{sohl2015deep}, encouraging the agent to focus more on later time steps. We repeat the experiments across five random seeds.

\paragraph{Results}  
We compare $\algseiko$ with $m = 4$ and $\eta_{\text{base}} = 2$, where the diffusion model is updated $\mathcal K = 4$ times, generating only a single final image per trajectory. We evaluate the final aesthetic score of the generated image following our learned $\hat\pi, \hat q$ and compare it to the one produced by \textsc{SEIKO}. Additionally, we compare the total training time of $\algseiko$ and \textsc{SEIKO}. To ensure a fair comparison, both algorithms are run under the same update time $\mathcal K$ and with the same total number of measurements, $N = 19200$. Our results are summarized in Figure~\ref{fig:seiko-base}. We report the mean and standard deviation of both reward and runtime over multiple seeds, demonstrating that $\algseiko$ achieves aesthetic scores comparable to \textsc{SEIKO} while requiring significantly fewer episodes. This reduction translates to approximately half the training time, demonstrating the efficiency of $\algseiko$.

\paragraph{Ablation Study}
To evaluate the impact of the number of policy updates $\mathcal K$ on both aesthetic reward and computational efficiency, we conduct an ablation study comparing $\algseiko$ with different $\eta_{\text{base}}$, setting $\mathcal K = 2, 4, 8$ under the same sample budget $N = 19200$. The results are summarized in Figure~\ref{fig:seiko-ablation1}. Our findings suggest that the number of policy updates exhibits a threshold at $\mathcal K = 4$. Specifically, for $\algseiko$ with $\mathcal K \geq 4$, the aesthetic score remains nearly unchanged even if $\mathcal K$ is reduced. However, when $\mathcal K$ is too small, e.g., $\mathcal K = 2$, the aesthetic score drops significantly. This empirical evidence supports our theoretical result in Theorem~\ref{thm:2}, which establishes a lower bound on the number of policy updates required.

Furthermore, we analyze the impact of the measurement frequency $m$ by comparing $\algseiko$ with $m = 1, 4, 40$. As illustrated in Figure~\ref{fig:seiko-ablation2}, increasing $m$ generally reduces the total training time by decreasing the total number of episodes. However, setting $m$ excessively high leads to performance degradation, suggesting that an optimal choice of $m$ is necessary to balance the quality of the generated images and training efficiency. This supports our claim in Theorem~\ref{thm:3}, which emphasizes the importance of selecting an appropriate $m$ for achieving the best trade-off.

\subsection{Continuous-Time Control} \label{exp:oderl}

\paragraph{Experiment Setup}
We study continuous-time control tasks in the standard Gym benchmark \citep{brockman2016openai}, focusing on three tasks: Acrobot, Pendulum, and CartPole. As our baseline model, the Ensemble Neural ODE (\textsc{ENODE}) \citep{yildiz2021continuous} is used. The dynamical system is deterministic, and the reward function is known. \textsc{ENODE} employs a low‐policy rollout strategy. The length of total observation time for a trajectory is $T=50$(s), and the measurement frequency is $m=250$. We defer details on \textsc{ENODE}’s backbone architecture, training procedure and the details of the sampler $\cT$ to Appendix~\ref{app:oderl_env}.

\paragraph{Algorithm Implementation}  
We implement $\algenode$ based on \textsc{ENODE}, incorporating the low-policy update strategy introduced in $\algpolicy$. Specifically, we adopt a batch-like strategy similar to \textsc{SEIKO}, as described in Section~\ref{exp:PURE-SEIKO}, to reduce the frequency of policy updates. In this approach, the batch size $B_i$ is doubled at each step, following $B_{i+1} = 2B_i$. To ensure the same sample budget $N$, we may slightly modify the doubling batch size strategy according to the environments. Further details on the algorithmic implementation, dataset collection, and policy updates for our experiments are provided in Appendix~\ref{app:oderl_exp_details}.

\paragraph{Results}
Our main results are presented in Figure~\ref{fig:oderl}. We report both the mean and standard derivation for both reward and running time from $20$ seeds. Empirically, $\algenode$ reaches the target state using only $1/2$ to $1/4$ of the policy update steps required by standard \textsc{ENODE}. This reduction also leads to nearly a $50\%$ decrease in training time. Such efficiency gains align with our theoretical predictions in the main theorems. To further evaluate the effectiveness of the policy update strategy, we conducted ablation study about the measurement frequency $m$ and different batch update scheduling strategy with varing policy updates. The corresponding results are deferred to Appendix~\ref{app:oderl_ablation_study}.

\section{Conclusion and Limitations}
\paragraph{Conclusion} In this work, we study CTRL with general function approximation under a finite number of measurements, aiming to reduce its computational cost. We develop a theoretical framework to propose our algorithm, $\algname$, along with a finite-sample analysis. Additionally, we introduce several variants of our algorithm with improved computational efficiency. Empirical results on continuous-time control tasks and fine-tuning of diffusion models backup our theoretical findings.

\paragraph{Limitations} 
While our analysis yields useful insights for sample- and computationally efficient CTRL, it does come with several caveats. First, the regret bound grows on the order of $\exp(T)$, which may render it vacuous for long time horizons. Second, our theoretical guarantees hinge on the Euler–Maruyama discretization \citep{platen2010numerical} of the underlying SDE, introducing unquantified bias from discretization error that could degrade performance in practice. Finally, we assume access to jointly measured observations $(x(t),x(t+\Delta))$ following \citet{treven2024efficient}, which may be infeasible in systems with asynchronous sensors or communication delays. We leave these challenges as directions for future work.

\bibliography{ref}

\newpage
\onecolumn
\appendix

\section{Proof of Propositions}

\subsection{Proof of Proposition \ref{smallelu}}\label{proof:prop1}

\begin{proof}
    Note that for any $p \in \cP$, $f \in \cF$, $b \in \cR$, we have
    \begin{align}
        \EE_{z \sim p}\|f(z) - f^*(z)\|_2^2 = \EE_{z \sim p}\phi(z)^\top (\Theta-\Theta^*)(\Theta-\Theta^*)^\top \phi(z) = \EE_{z \sim p}\langle (\Theta-\Theta^*)(\Theta-\Theta^*)^\top, \phi(z)\phi(z)^\top\rangle,\notag
    \end{align}
    and similarily, 
    \begin{align}
        \EE_{z \sim p}(b(z) - b^*(z))^2 = \EE_{z \sim p}\langle (\theta - \theta^*)(\theta - \theta^*)^\top, \phi(z)\phi(z)^\top\rangle.\notag
    \end{align}
    Therefore, $d_{\cF,\epsilon}$ can be bounded by $\operatorname{DE}_1(\cZ,\bar\cF,\cP,\epsilon)$, where $\bar\cF = \{f_{\theta}(z) = \langle \bar\Theta, \bar\phi(z)\rangle: \|\bar\Theta\|_F \leq R^2, \|\bar\phi\|_2 \leq L^2, \phi \in \RR^{d^2}\}$. Therefore, first through Lemma 5.4 in \citet{wang2023benefits}, we have $\operatorname{DE}_1(\cZ,\bar\cF,\cP,\epsilon) \leq \operatorname{DE}_2(\cZ,\bar\cF,\cP,\epsilon)$. Then according to Proposition 29 in \citet{jin2021bellman}, we have $\operatorname{DE}_2(\cZ,\bar\cF,\cP,\epsilon) = O(d^2\log(1+R^2L^2/\epsilon^2))$ for any $\cP$. Similar arguments hold for $\cR$ as well. 
\end{proof}

\subsection{Proof of Proposition \ref{prop:22}}\label{proof:prop2}
\begin{proof}
    Consider a sampler $\mathcal{T}$ that selects time points $t_1, \dots, t_m$ uniformly in $[0,T]$, i.e., $t_i = \frac{i}{m} \cdot T$. We consider an  one-dimensional Ornstein–Uhlenbeck (OU) process follows 
$$
dx(t) = -u \cdot x(t) dt + \sqrt{2} dw(t), \quad x(0) = 0,
$$
where $u_{\min} \leq u \leq u_{\max}$. Our dynamic class $\mathcal{F}$ is a singleton, and the reward class is defined as $\mathcal{R} = \{ b(x,u) = \alpha \cdot x \mid 0 \leq \alpha \leq 1 \}$ with $b^*(x,u) = x$. We have $C_{\mathcal{T},m,\mathcal{F},i} = 1$. To bound $C_{\mathcal{T},m,\mathcal{R},i}$, note that $(b - b^*)^2 = (1-\alpha)^2x^2$. For simplicity, we use $z$ to denote $x$ and omit $u$ since it is a constant control unit. Then the Gaussian marginal distribution of the OU process gives
$$
\mathbb{P}(\hat{Z} = z) = \frac{1}{T} \int_{0}^{T} 
\sqrt{\frac{u}{2\pi (1 - e^{-2u t})}} 
\exp\left( -\frac{u z^2}{2 (1 - e^{-2u t})} \right) dt.
$$
Since the OU process is a Markov process, we have 
$$
\mathbb{P}(Z(t_i) = z \mid Z(t_{i-1}) = z_{i-1}, \dots, Z(t_1) = z_1) = \mathbb{P}(Z(t_i) = z \mid Z(t_{i-1}) = z_{i-1}).
$$
The transition density is given by
$$
\mathbb{P}(Z(t_i) = z \mid Z(t_{i-1}) = z_{i-1}) =
\sqrt{\frac{u}{2\pi (1 - e^{-2 u T/m})}}
\exp\left(
-\frac{u (z - z_{i-1} e^{-u T/m})^2}
{2 (1 - e^{-2 u T/m})}
\right).
$$
We compute the expectations over $(b-b^*)^2$. For $\hat{Z}$,
\begin{align}
(1-\alpha)^{-2} \mathbb{E}[(b - b^*)^2(\hat{Z})] &= \mathbb{E}[\hat{Z}^2] = \mathbb{E}[\mathbb{E}[Z_t^2 \mid t]] = \int_{0}^{T} \frac{1}{u} (1 - e^{-2u t}) \frac{1}{T} dt = \frac{1}{u} - \frac{1}{2 u^2 T} (1 - e^{-2 u T}).\notag
\end{align}
For $Z(t_i)$,
\begin{align}
(1-\alpha)^{-2} \mathbb{E}[(b(Z(t_i)) - b^*(Z(t_i)))^2  \mid Z(t_{i-1}) = z_{i-1}] &=\mathbb{E}[Z(t_i)^2 \mid Z(t_{i-1}) = z_{i-1}] \notag \\
&= (z_{i-1} e^{-u T/m} )^2 + \frac{1}{u} (1 - e^{-2 u T/m}).\notag
\end{align}
Then the ratio is bounded by
\begin{align}
\frac{\mathbb{E}[(b - b^*)^2(\hat{Z})]}{\mathbb{E}[(b(Z(t_i)) - b^*(Z(t_i)))^2 \mid Z(t_{i-1}) = z_{i-1}]} 
&\leq \frac{\frac{1}{u} - \frac{1}{2 u^2 T} (1 - e^{-2 u T})}{\frac{1}{u} (1 - e^{-2 u T/m})} \notag \\
&\leq \frac{1}{1 - e^{-2 u T/m}} \notag \\
&\leq 1 + \frac{m}{2 T u_{\min}},\notag
\end{align} 
which implies $C_{\cT,m,\cR,i} \leq 1 + \frac{m}{2 T u_{\min}}$ for all $i$. 
Thus, we obtain the upper bound for $C_{\cT,m\,\cR}$. 
\end{proof}

\section{Proof of Theorem \ref{thm:1}}\label{proof:thm:1}
In this section we prove Theorem \ref{thm:1}. To make our presentation more clear, we separate Theorem \ref{thm:1} into two theorems Theorem \ref{thm:1:conf} and Theorem \ref{thm:1:regret} and prove them separately. To begin with, we have the following lemma to bound the flow first. 
\begin{lemma}\label{lemma:flow}
Denote $\hat x(t)$ to be the state flow that following $f_n, \pi_n, q_n$, and let $x(t)$ denote the state flow following $f^*, \pi_n, q_n$. Then we have
    \begin{align}
        \EE\|\hat x_n(t) - x_n(t)\|_2^2 \leq 2 e^{Kt}\cdot \int_{s=0}^t \EE[\|f^*(x_n(s), \pi_n(x_n(s))) - f_n(x_n(s), \pi_n(x_n(s)))\|_2^2]ds. \notag
    \end{align}
    where $K = 1 + d\cdot (1+L_\pi)^2\cdot L_g^2 + 2(1+L_\pi)^2\cdot L_f^2.$
\end{lemma}
\begin{proof}
 Define \(\delta(t) := \hat{x}_n(t) - x_n(t)\). The dynamics of \(\delta(t)\) are governed by:
 \begin{align}
 d\delta(t) &= \left[f_n(\hat{x}_n(t), \pi_n(\hat{x}_n(t))) - f^*(x_n(t), \pi_n(x_n(t)))\right]dt \notag\\&\quad + \left[g^*(\hat{x}_n(t), \pi_n(\hat{x}_n(t))) - g^*(x_n(t), \pi_n(x_n(t)))\right]dw(t).\label{mart_term}
 \end{align}
 Applying Itô's lemma to \(\|\delta(t)\|_2^2\) yields:
 \[
 d\|\delta(t)\|_2^2 = 2\delta(t)^\top d\delta(t) +d\cdot |g^*(\hat{x}_n(t), \pi_n(\hat{x}_n(t))) - g^*(x_n(t), \pi_n(x_n(t)))|^2 dt.
 \]
 Taking integration on both sides:
 \begin{align*}
 \|\delta(t)\|_2^2 = \|\delta(0)\|_2^2+\int_0^t2\delta(t)^\top d\delta(t) +d\cdot \int_0^t |g^*(\hat{x}_n(t), \pi_n(\hat{x}_n(t))) - g^*(x_n(t), \pi_n(x_n(t)))|^2 dt.
 \end{align*}
 Taking expectations, the martingale term corresponding to (\ref{mart_term}) vanishes, then apply Fubini's theorem for other terms leading to:
 \begin{align*}
 \frac{d}{dt}\mathbb{E}\|\delta(t)\|_2^2 &= 2\mathbb{E}\left[\delta(t)^\top\left(f_n(\hat{x}_n(t), \pi_n(\hat{x}_n(t))) - f^*(x_n(t), \pi_n(x_n(t)))\right)\right] \\&\quad + d\cdot \mathbb{E}|g^*(\hat{x}_n(t), \pi_n(\hat{x}_n(t))) - g^*(x_n(t), \pi_n(x_n(t)))|^2.
 \end{align*}
 Bounding the first term using Cauchy-Schwarz and $2ab\leq a^2+b^2$:
 \begin{align*}
 & 2\mathbb{E}\left[\delta(t)^\top\left(f_n(\hat{x}_n(t), \pi_n(\hat{x}_n(t))) - f^*(x_n(t), \pi_n(x_n(t)))\right)\right]
 \\&\leq 2\mathbb{E}\left[\|\delta(t)\|_2 \cdot \|f_n(\hat{x}_n(t), \pi_n(\hat{x}_n(t))) - f^*(x_n(t), \pi_n(x_n(t)))\|_2\right]
 \\ &\leq  \mathbb{E}[\|\delta(t)\|_2^2] + \mathbb{E}[\left\|f_n(\hat{x}_n(t), \pi_n(\hat{x}_n(t))) - f^*(x_n(t), \pi_n(x_n(t)))\right\|_2^2].
 \end{align*}
 To leverage the \(L_f\)-Lipschitzness of \(f_n, f^*\) (Assumption \ref{Lipshitzness}), we bound the term:
 \begin{align*}
 &\mathbb{E}\left\|f_n(\hat{x}_n(t), \pi_n(\hat{x}_n(t))) - f^*(x_n(t), \pi_n(x_n(t)))\right\|_2^2
 \\ &\leq 2\mathbb{E}[\left\|f_n(\hat{x}_n(t), \pi_n(\hat{x}_n(t))) - f_n(x_n(t), \pi_n(x_n(t)))\right\|_2^2] + 2\mathbb{E}[\left\|f_n(x_n(t), \pi_n(x_n(t))) - f^*(x_n(t), \pi_n(x_n(t)))\right\|_2^2]
 \\ &\leq 2L_f^2 (1 + L_\pi)^2 \mathbb{E}[\|\delta(t)\|_2^2]+2\mathbb{E}[\left\|f_n(x_n(t), \pi_n(x_n(t))) - f^*(x_n(t), \pi_n(x_n(t)))\right\|_2^2],
 \end{align*}
 where the first inequality is due to the fact that $\|a + b\|^2 \leq 2\|a\|^2 + 2\|b\|^2$, the second inequality is due to Assumption \ref{Lipshitzness}.

 Similarly for the diffusion term, by \(g^*\) is \(L_g\)-Lipschitz:
 \[
 \mathbb{E}|g^*(\hat{x}_n(t), \pi_n(\hat{x}_n(t))) - g^*(x_n(t), \pi_n(x_n(t)))|^2 \leq (1+L_\pi)^2\cdot L_g^2 \mathbb{E}\|\delta(t)\|_2^2.
 \]

 Combining above:
 \begin{align*}
 \frac{d}{dt}\mathbb{E}\|\delta(t)\|_2^2  \leq (1 + d\cdot (1+L_\pi)^2\cdot L_g^2 + 2(1+L_\pi)^2\cdot L_f^2)\mathbb{E}\|\delta(t)\|_2^2 + 2\mathbb{E}\left\|f_n(x_n(t), \pi_n(x_n(t))) - f^*(x_n(t), \pi_n(x_n(t)))\right\|_2^2.
 \end{align*}
 Applying Gr\"onwall's inequality with \(K = 1 + (1+L_\pi)^2\cdot L_g^2 + 2(1+L_\pi)^2\cdot L_f^2\):
 \[
 \mathbb{E}\|\delta(t)\|_2^2 \leq 2e^{K t} \int_0^t \mathbb{E}\left\|f_n(x_n(s), \pi_n(x_n(s))) - f^*(x_n(s), \pi_n(x_n(s)))\right\|_2^2 ds.
 \]
 \end{proof}

We also have the following lemma to show the difference between $b-b^*$ with their empirical one:
\begin{lemma}\label{lemma:empto}[Lemma 1,5, \citealt{russo2013eluder}]
For any $\delta>0$ and $\epsilon>0$, let $\Phi_\epsilon$ be the $\epsilon$-covering net of $\Phi$ (for any $\phi \in \Phi$, there exists a $\phi_\epsilon \in \Phi_\epsilon$ such that $\|\phi - \phi_\epsilon\|_\infty \leq \epsilon$). Then with probability at least $1-\delta$, we have
\begin{itemize}
    \item For all $b \in \cR$, recall that $|b| \leq 1$ and $r$ is 1-Gaussian, then we have for all $n$,
    \begin{align}
     &\sum_{z,r \in \cD_n} (b(z) - r)^2  - (b^*(z) - r)^2\geq \frac{1}{2}\sum_{z \in \cD_n}(b(z) - b^*(z))^2 -4\log(|\cR_\epsilon|/\delta) - \epsilon n(8 + \sqrt{8\log(4n^2|\cR_\epsilon|/\delta)}),\notag\\
     &\sum_{z,r \in \cD_n} (b(z) - r)^2  - (b^*(z) - r)^2\leq \frac{3}{2}\sum_{z \in \cD_n}(b(z) - b^*(z))^2 +4\log(|\cR_\epsilon|/\delta) + \epsilon n(\epsilon+8 + \sqrt{8\log(4n^2|\cR_\epsilon|/\delta)}),\notag
\end{align}
\item For all $f \in \cF$, recall that $\|f\|_2 \leq 1$ and $y$ is $g(z)\cdot \textbf{1}$-Gaussian vector, $|g| \leq G/\sqrt{d}$, then we have for all $n$,
\begin{align}
    \sum_{z \in \cD_n}\|f(z) - f^*(z)\|^2 &\leq 2\bigg( \sum_{z,y \in \cD_n} \|f(z) - y\|^2  - \sum_{z,y \in \cD_n} \|f^*(z) - y\|^2\bigg) \notag\\&\quad + 8G^2\log(|\cF_\epsilon|/\delta) + 2\epsilon n(8 + \sqrt{8G^2\log(4n^2|\cF_\epsilon|/\delta)}).\notag
\end{align}
\begin{align}
    \sum_{z,y \in \cD_n} \|f(z) - y\|^2  - \|f^*(z) - y\|^2 &\leq \frac{3}{2}\sum_{z \in \cD_n}\|f(z) - f^*(z)\|^2  + 4G^2\log(|\cF_\epsilon|/\delta)\notag\\&\quad + \epsilon n(\epsilon+8 + \sqrt{8G^2\log(4n^2|\cF_\epsilon|/\delta)}).\notag
\end{align}
\end{itemize}

\end{lemma}
\begin{proof}
To prove the first part, let \(\cR_\epsilon\subset \cR\) be a finite set such that for every \(b \in \cR\) there exists some \(b_\epsilon \in \cR_\epsilon\) with \(\|b - b_\epsilon\|_\infty \le \epsilon\).

Recall the notations \(L_{2,n}(b):=\sum_{z,r\in\cD_n}(b(z)-r)^2\) and \(\|b - b^*\|_{2,\cD_n}^2 := \sum_{z\in\cD_n}(b(z)-b^*(z))^2\) from \citealt{russo2013eluder}. Following the proof of Lemma 1 of \citealt{russo2013eluder}, for each \(b_\epsilon \in \cR_\epsilon\), with \(\eta^2 = 1\) as our 1-Gaussian assumption, with probability at least \(1-\frac{\delta}{|\cR_\epsilon|}\):
   \begin{align}
     \|b_\epsilon - b^*\|_{2,\cD_n}^2 
     \le
     2\left(L_{2,n}(b_\epsilon) 
              -L_{2,n}(b^*)\right)
     +8\ln\left(\frac{|\cR_\epsilon|}{\delta}\right),\label{rvlemma1}
   \end{align}
    taking union bound over \(\cR_\epsilon\) we then have with probability at least \(1 - \delta\), every \(b_\epsilon \in \cR_\epsilon\) satisfies above inequality.

Now let \(b\in \cR\) be arbitrary. By the definition of the \(\epsilon\)-net, pick \(b_\epsilon \in \cR_\epsilon\) with \(\|b - b_\epsilon\|_\infty \le \epsilon\).

Since \(\|b - b_\epsilon\|_\infty \le \epsilon\) and \(\|b\|_\infty \le 1\) and the reward noise variance is \(\eta^2=1\), apply Lemma 5 of \citealt{russo2013eluder} and taking union bound we have with probability at least \(1-\delta\) for all \(n\in \mathbb{N}\):
\begin{align}
  \left|
     \frac{1}{2}\|b_\epsilon - b^*\|_{2,\cD_n}^2 
     -\frac{1}{2}\|b - b^*\|_{2,\cD_n}^2 
     +L_{2,n}(b) 
     -L_{2,n}(b_\epsilon)
  \right|
  \le
  \epsilon n\left[8 + \sqrt{8\ln\left(\frac{4n^2|\cR_\epsilon|}{\delta}\right)}\right],\label{rvlemma5}\end{align}

We now connect \(\|b - b^*\|^2\) to \(\|b_\epsilon - b^*\|^2\).  By (\ref{rvlemma5}):

\[
\begin{aligned}
  \frac{1}{2}\|b - b^*\|_{2,\cD_n}^2
  &=
  \frac{1}{2}\|b_\epsilon - b^*\|_{2,\cD_n}^2 
  -(L_{2,n}(b_\epsilon) - L_{2,n}(b))
  +\left[\frac{1}{2}\|b - b^*\|^2 - \frac{1}{2}\|b_\epsilon - b^*\|^2 + L_{2,n}(b_\epsilon) - L_{2,n}(b)\right]
  \\
  &\le
  \frac{1}{2}\|b_\epsilon - b^*\|_{2,\cD_n}^2
  -(L_{2,n}(b_\epsilon) - L_{2,n}(b))
  +\epsilon n\left[8 + \sqrt{8\ln\left(\frac{4n^2|\cR_\epsilon|}{\delta}\right)}\right].
\end{aligned}
\]
Multiply both sides by \(2\):
\[
  \|b - b^*\|_{2,\cD_n}^2
  \le
  \|b_\epsilon - b^*\|_{2,\cD_n}^2
  -2(L_{2,n}(b_\epsilon) - L_{2,n}(b))
  +2\epsilon n\left[8 + \sqrt{8\ln\left(\frac{4n^2|\cR_\epsilon|}{\delta}\right)}\right].
\]
Then apply (\ref{rvlemma1}) for \(\|b_\epsilon - b^*\|_{2,\cD_n}^2\):

\[
\begin{aligned}
  \|b - b^*\|_{2,\cD_n}^2
  &\le
  2(L_{2,n}(b_\epsilon) - L_{2,n}(b^*))
  +8\ln\left(\tfrac{|\cR_\epsilon|}{\delta}\right)
  -2(L_{2,n}(b_\epsilon) - L_{2,n}(b))
  +2\epsilon n\left[8 + \sqrt{8\ln\left(\frac{4n^2|\cR_\epsilon|}{\delta}\right)}\right]
  \\
  &=
  2(L_{2,n}(b) - L_{2,n}(b^*) )
  +8\ln\left(\tfrac{|\cR_\epsilon|}{\delta}\right)
  +2\epsilon n\left[8 + \sqrt{8\ln\left(\frac{4n^2|\cR_\epsilon|}{\delta}\right)}\right]
  \\
  &=
  2\Bigl(\sum_{z,r}(b - r)^2 - \sum_{z,r}(b^*(z) - r)^2\Bigr)
  \;+\;8\,\ln\!\Bigl(\tfrac{|\cR_\epsilon|}{\delta}\Bigr)
  \;+\;2\epsilon n\left[8 + \sqrt{8\ln\left(\frac{4n^2|\cR_\epsilon|}{\delta}\right)}\right].
\end{aligned}
\]
This proves the first inequality. To prove the second inequality, first note that by modifying the proof of Lemma 1 in \citealt{russo2013eluder}, namely setting $Z_t:=f(A_t)-R_t)^2-(f_\theta(A_t)-R_t)^2$ which is the negative of original $Z_t$, we can go through the same arguments as the proof in \citealt{russo2013eluder} and arrive at the conclusion that with probability at least $1-\delta$, for every $b_\epsilon\in\cR_\epsilon$ we have:
$$ L_{2,n}(b_\epsilon)-L_{2,n}(b^*)\leq\|b_\epsilon-b^*\|_{2,\cD_n}^2+4\eta^2\ln\left(\frac{|\cR_\epsilon|}{\delta}\right).$$

To combine this with Lemma 5 of \citealt{russo2013eluder}, we decompose as:
\begin{align*}
    L_{2,n}(b)-L_{2,n}(b^*)&=(L_{2,n}(b)-L_{2,n}(b_\epsilon))+(L_{2,n}(b_\epsilon)-L_{2,n}(b^*))\\
    &=\frac{1}{2}\|b - b^*\|_{2,\cD_n}^2-\frac{1}{2}\|b_\epsilon - b^*\|_{2,\cD_n}^2+\left[\frac{1}{2}\|b_\epsilon - b^*\|^2-\frac{1}{2}\|b - b^*\|^2  + L_{2,n}(b)-L_{2,n}(b_\epsilon) \right]\\
  &\quad +(L_{2,n}(b_\epsilon)-L_{2,n}(b^*)) \\
  &\leq \frac{1}{2}\|b - b^*\|_{2,\cD_n}^2+\epsilon n\left[8 + \sqrt{8\ln\left(\frac{4n^2|\cR_\epsilon|}{\delta}\right)}\right]+\frac{1}{2}\|b_\epsilon-b^*\|_{2,\cD_n}^2+4\ln\left(\frac{|\cR_\epsilon|}{\delta}\right).
\end{align*}

To bound $\|b_\epsilon-b^*\|_{2,\cD_n}^2$, note that $(b_\epsilon-b^*)^2\leq 2(b_\epsilon-b)^2+2(b-b^*)^2$ and thus we can bound as:

$$ \frac{1}{2}\|b_\epsilon-b^*\|_{2,\cD_n}^2\leq n\epsilon^2+\|b - b^*\|_{2,\cD_n}^2.$$

Plug in and we obtain the bound:

$$L_{2,n}(b)-L_{2,n}(b^*)\leq \frac{3}{2}\|b - b^*\|_{2,\cD_n}^2+\epsilon n\left[8 + \sqrt{8\ln\left(\frac{4n^2|\cR_\epsilon|}{\delta}\right)}\right]+n\epsilon^2+4\ln\left(\frac{|\cR_\epsilon|}{\delta}\right).$$

Thus we prove the first part of the lemma. The proof for the second part of the lemma is by using the same arguments, except replacing the noise variance with $\eta^2=G^2$, which is simply by adding the extra $G^2$ coefficients to the bound.    
\end{proof}

\begin{lemma}[Theorem 5.3, \citealt{wang2023benefits}]\label{lemma:runzhe}
        Given a function class $\Phi$ defined on $\cZ$ with $|\phi(x)| \leq 1$ for all $(\phi, z) \in \Phi \times \cZ$, and a family of probability measures $\cP$ over $\cZ$. Suppose sequence $\{\phi_k\}_{k=1}^K \subset \Phi$ and $\{p_k\}_{k=1}^K \subset \cP$ satisfy that for all $k \in [K]$, $\sum_{i=1}^{k-1} |\mathbb{E}_{p_i}[\phi_k]| \leq \beta$. Then for all $k \in [K]$, 
\[
\sum_{i=1}^k |\mathbb{E}_{p_i}[\phi_i]| \leq O(\text{DE}_1(\cZ, \Phi, \cP, 1/k)\beta\log k)
\]
\end{lemma}

Next we are going to prove our theorems. 
\begin{theorem}\label{thm:1:conf}
Let
\begin{align}
    &\beta_\cR:=8\log(|\cR_\epsilon|/\delta) + 2\epsilon N(8 + \sqrt{8\log(4N^2|\cR_\epsilon|/\delta)})\notag\\
    &\beta_\cF:=8G^2\log(|\cF_\epsilon|/\delta) + 2\epsilon N(8 + \sqrt{8G^2\log(4N^2|\cF_\epsilon|/\delta)}),\notag
\end{align}
    then we have $f^* \in \cF_n$ and $b^* \in \cR_n$ for all $n$ w.h.p. 
\end{theorem}
\begin{proof}
Recall the definition of $\cR_n$ and $\cF_n$: 
    \begin{align}
    \cR_{n}\leftarrow \bigg\{b:\sum_{z,y,r \in \cD_{n}}(b(z) - r)^2 \leq \min_{b' \in \cR} \sum_{z,y,r \in \cD_{n}} (b'(z) - r)^2 + \beta_{\cR}\bigg\}\notag \\
    \cF_{n}\leftarrow \bigg\{f:\sum_{z,y,r \in \cD_{n}}\|f(z) - y\|^2 \leq \min_{f' \in \cF} \sum_{z,y,r \in \cD_{n}} \|f'(z) - y\|^2 + \beta_{\cF}\bigg\}. \notag
\end{align}
Following Lemma \ref{lemma:empto}, we have that for any $b \in \cR_n$, we have
\begin{align}
&\sum_{z,y,r \in \cD_{n}}(b^*(z) - r)^2 -  \sum_{z,y,r \in \cD_{n}} (b(z) - r)^2 \notag \\
&\leq -\frac{1}{2}\sum_{z \in \cD_n}(b(z) - b^*(z))^2 + 8\log(|\cR_\epsilon|/\delta) + 2\epsilon n(8 + \sqrt{8\log(4n^2|\cR_\epsilon|/\delta)})\notag \\
& \leq 8\log(|\cR_\epsilon|/\delta) + 2\epsilon n(8 + \sqrt{8\log(4n^2|\cR_\epsilon|/\delta)})\notag \\
& \leq \beta_{\cR} \notag
\end{align}
following the definition of $\beta_{\cR}$. Therefore, we have $b^* \in \cR_n$. Similarily, we have $f^* \in \cF_n$. 
\end{proof}

We have the following theorem.

\begin{theorem}\label{thm:1:regret}
With probability at least $1-2\delta\log N$, we have
\begin{align}
\sum_{n=1}^N R_n &=O(T\sqrt{Nd_{\cR}( \log(4N/\delta)+\log(|\cR_\epsilon|/\delta))\log N} \notag\\ &\quad + LT\cdot\sqrt{TN\cdot 2\exp(K T)}\sqrt{d_{\cF}G^2(\log(4N/\delta)+\log(|\cF_\epsilon|/\delta))\log N}),\notag
\end{align}    
where $\epsilon = \frac{1}{N^2},K=1 + d\cdot (1+L_\pi)^2\cdot L_g^2 + 2(1+L_\pi)^2\cdot L_f^2,L=L_b(1+L_\pi).$
\end{theorem}

\begin{proof}
    For simplicity, denote $\hat x(t)$ to be the state flow that following $f_n, \pi_n, q_n$, and let $x(t)$ denote the state flow following $f^*, \pi_n, q_n$. We introduce the notion of $R(f,b,\pi,q)$ to denote
\begin{align}
    R(f,b,\pi, q):=\EE\bigg[\int_{t=0}^T b(x(t), \pi(x(t)))dt\bigg|x(0)\sim q, dx(t) = f(x(t), \pi(x(t))dt + g^*(x(t), \pi(x(t))dw(t)\bigg],\notag
\end{align}
Then we have
\begin{align}
R_n: &= R(f^*, r^*,\pi^*, q^*) - R(f^*, r^*,\pi_n, q_n) \notag \\
&\leq R(f_n, b_n, \pi_n, q_n) - R(f^*, r^*,\pi_n, q_n)\notag \\
& = \EE\bigg[\int_{t=0}^T b_n(\hat x(t), \pi_n(\hat x(t)))dt\bigg] - \EE\bigg[\int_{t=0}^T b^*(x(t), \pi_n(x(t)))dt\bigg]\notag \\
& = \EE\bigg[\int_{t=0}^T b_n(\hat x(t), \pi_n(\hat x(t)))dt\bigg] - \EE\bigg[\int_{t=0}^T b_n(x(t), \pi_n(x(t)))dt\bigg] + \EE\bigg[\int_{t=0}^T (b_n - b^*)(x(t), \pi_n(x(t)))dt\bigg] \notag \\
& \leq L_b(1+L_\pi)\cdot \EE\bigg[\int_{t=0}^T\|\hat x(t) - x(t)\|_2 dt\bigg] + \EE\bigg[\int_{t=0}^T (b_n - b^*)(x(t), \pi_n(x(t)))dt\bigg],\notag
\end{align}
where the first inequality holds since $f^*\in \cF_n, b^*\in \cR_n$ and the optimism principle, and the last one holds due to Assumption \ref{Lipshitzness}. 
By Lemma \ref{lemma:flow}, we have
    \begin{align}
        \EE\|\hat x_n(t) - x_n(t)\|_2^2 \leq 2\exp(K t)\cdot \int_{s=0}^t \EE[\|f^*(x_n(s), \pi_n(x_n(s))) - f_n(x_n(s), \pi_n(x_n(s)))\|_2^2] ds. \notag
    \end{align}

    Therefore, we have
\begin{align}
    R_n &\leq L\cdot \bigg[\int_{t=0}^T\sqrt{2 \exp(K t)\cdot \int_{s=0}^t \EE[\|f^*(x_n(s), \pi_n(x_n(s))) - f_n(x_n(s), \pi_n(x_n(s)))\|_2^2]} dt\bigg] \notag \\
    &\quad+ \EE\bigg[\int_{t=0}^T (b_n - b^*)(x(t), \pi_n(x(t)))dt\bigg] \notag \\
    & \leq LT \bigg[\sqrt{2 \exp(K T)\cdot \int_{s=0}^T \EE[\|f^*(x_n(s), \pi_n(x_n(s))) - f_n(x_n(s), \pi_n(x_n(s)))\|_2^2ds]} \bigg] + T\sqrt{B_n}\notag \\
    & = LT \sqrt{2T \exp(K T)}\sqrt{A_n} + T\sqrt{B_n}, \label{help:boundz}
\end{align}
where 
\begin{align}
    A_n: = \EE_{x_n, t}\|f^*(x_n(t), \pi_n(x_n(t))) - f_n(x_n(t), \pi_n(x_n(t)))\|_2^2 ,\notag\\
    B_n: = \EE_{x_n, t} |b_n(x_n(t), \pi_n(x_n(t)))- b^*(x_n(t), \pi_n(x_n(t)))|^2.\notag
\end{align}
Here, the second inequality holds due to the basic inequality $\EE[x] \leq \sqrt{\EE[x^2]}$. Taking summation from $n =1$ to $N$, we have
\begin{align}
    \sum_{n=1}^N R_n 
    & \leq \sum_{n=1}^N LT \sqrt{2T \exp(K T)}\sqrt{A_n} + T\sqrt{B_n}  \notag\\
    & \leq T\sqrt{N\sum_{n=1}^N B_n} + LT\cdot\sqrt{TN\cdot 2\exp(K T)}\sqrt{\sum_{n=1}^N A_n},\label{proof:1}
\end{align}

Let $p_n$ denote the distribution of $z_n$, where $z_n = (x_n(t), \pi_n(x_n(t)))$ with the following joint distribution:
\begin{align}
    t \sim \text{Unif}[0, T], x\sim X(t, \pi_n, q_n). \notag
\end{align}
With a slight abuse of notation, we use $f(z_n), b(z_n)$ to denote $f(x_n(t), \pi_n(x_n(t))), b(x_n(t), \pi_n(x_n(t)))$. Next we just need to make sure that for both $f$ and $b$, they satisfy the Eluder dimension. First, note that to train $b_n$ and $f_n$, we obtain it from the following one:
\begin{align}
    b_n \in \cR_n, \cR_{n}\leftarrow \bigg\{b:\sum_{z,y,r \in \cD_{n}}(b(z) - r)^2 \leq \min_{b' \in \cR} \sum_{z,y,r \in \cD_{n}} (b'(z) - r)^2 + \beta_{\cR}\bigg\}\notag \\
    f_n \in \cF_n, \cF_{n}\leftarrow \bigg\{f:\sum_{z,y,r \in \cD_{n}}\|f(z) - y\|^2 \leq \min_{f' \in \cF} \sum_{z,y,r \in \cD_{n}} \|f'(z) - y\|^2 + \beta_{\cF}\bigg\}.\notag
\end{align}

First, we have for any $b \in \cR$ and $f \in \cF$, by Lemma \ref{crhs1} and taking union bound over $n$, for all $n$ we have with probability at least $1-2\delta \log N$:
\begin{align}
 \sum_{i=1}^{n-1} \EE_{z' \sim p_i}(b(z') - b^*(z'))^2 \leq 8 \sum_{z \in \cD_n} (b(z) - b^*(z))^2+ 4\log(4N/\delta),\label{help:11}\\
 \sum_{i=1}^{n-1} \EE_{z' \sim p_i}\|f(z') - f^*(z')\|^2 \leq 8 \sum_{z \in \cD_n} \|f(z) - f^*(z)\|^2+ 4\log(4N/\delta),\label{help:12}
\end{align}
Then taking $b = b_n, f = f_n$ and using Lemma \ref{lemma:empto}, we have
\begin{align}
    \sum_{i=1}^{n-1} \EE_{z' \sim p_i}(b_n(z') - b^*(z'))^2&\leq O\bigg( \sum_{z,r \in \cD_n} (b_n(z) - r)^2  - \inf_{b \in \cR}\sum_{z,r \in \cD_n} (b(z) - r)^2 + \log(4N/\delta)+\log(|\cR_\epsilon|/\delta)\bigg) \notag \\
    & \leq O(\beta_{\cR} + \log(4N/\delta)+\log(|\cR_\epsilon|/\delta)), \label{help:1}
\end{align}
and
\begin{align}
    \sum_{i=1}^{n-1} \EE_{z' \sim p_i}\|f(z') - f^*(z')\|^2 &\leq O\bigg(\sum_{z,y \in \cD_n} \|f_n(z) - y\|^2  - \inf_{f \in \cF}\sum_{z,y \in \cD_n} \|f_n(z) - y\|^2 + \log(4N/\delta)+\log(|\cF_\epsilon|/\delta)\bigg)\notag \\
    & \leq O(\beta_{\cF} +\log(4N/\delta)+ G^2\log(|\cF_\epsilon|/\delta)). \label{help:2}
\end{align}
Therefore, taking $\phi_i(z) = (b_i(z) - b^*(z))^2$ and $\phi_i(z) = \|f_i(z) - f^*(z)\|^2$ separately, we can use Lemma \ref{lemma:runzhe} for both cases and obtain that
\begin{align}
    &\sum_{n=1}^N B_n \leq O(d_{\cR}(\beta_{\cR} + \log(4N/\delta)+\log(|\cR_\epsilon|/\delta))\log N) = O(d_{\cR}\beta_{\cR}\log N),\notag \\
    &\sum_{n=1}^N A_n \leq O(d_{\cF}(\beta_{\cF} + \log(4N/\delta)+G^2\log(|\cF_\epsilon|/\delta))\log N) = O(d_{\cF}\beta_{\cF}\log N).\notag
\end{align}
Substituting them into \eqref{proof:1} finalizes our proof.

\end{proof}

\section{Proof of Theorem \ref{thm:2}}\label{app:thm2}
\begin{theorem}\label{thm:2:regret}
    With high probability, $f^*\in \cF_n, b^* \in \cR_n$. Meanwhile, the regret of Algorithm \ref{alg:alg2} is in the same order of Algorithm \ref{alg:alg1}. 
\end{theorem}
\begin{proof}
Define the following confidence sets:
\begin{align}
    \hat\cF_{n+1}\leftarrow \bigg\{f: \sum_{x,u,y,r \in \cD_{n+1}}(f(x,u) - y)^2 \leq \min_{f' \in \cF} \sum_{x,u,y,r \in \cD_{n+1}} (f'(x,u) - y)^2 + 5\beta_{\cF}\bigg\}.\notag\\
        \hat\cR_{n+1}\leftarrow \bigg\{b:\sum_{x,u,y,r \in \cD_{n+1}}(b_n(x,u) - r)^2 \leq \min_{b' \in \cR} \sum_{x,u,y,r \in \cD_{n+1}} (b'(x,u) - r)^2 + 5\beta_{\cR}\bigg\}. \notag
\end{align}
First, it is easy to see that Theorem \ref{thm:1:conf} still holds, therefore $b^* \in \cR_n \subset \hat \cR_n$ and $f^* \in \cF_n \subset \hat \cF_n$. Next, by our updating rule, we have for all $n$, 
\begin{align}
&\sum_{x,u,y,r \in \cD_{n+1}}(f_n(x,u) - y)^2 \leq \min_{f' \in \cF} \sum_{x,u,y,r \in \cD_{n+1}} (f'(x,u) - y)^2 + 5\beta_{\cF}\notag \\
&\sum_{x,u,y,r \in \cD_{n+1}}(b_n(x,u) - r)^2 \leq \min_{b' \in \cR} \sum_{x,u,y,r \in \cD_{n+1}} (b'(x,u) - r)^2 + 5\beta_{\cR}.\notag
\end{align}
Therefore, we can follow the proof of Theorem \ref{thm:1:regret} by changing $\beta_{\cR}$ and $\beta_{\cF}$ with $5\beta_{\cR}$ and $5\beta_{\cF}$ in \eqref{help:1} and \eqref{help:2}, the regret still holds. 
\end{proof}

\begin{theorem}
    The total switching number is $O(d_{\cF}\log N + d_{\cR}\log N)$.  
\end{theorem}
\begin{proof}
To begin with, note that for all $n$, we have
\begin{align}
&\sum_{x,u,y,r \in \cD_{n}}(f_n(x,u) - y)^2 \leq \min_{f' \in \cF} \sum_{x,u,y,r \in \cD_{n}} (f'(x,u) - y)^2 + 5\beta_{\cF}\notag \\
&\sum_{x,u,y,r \in \cD_{n}}(b_n(x,u) - r)^2 \leq \min_{b' \in \cR} \sum_{x,u,y,r \in \cD_{n}} (b'(x,u) - r)^2 + 5\beta_{\cR}.\notag
\end{align}
Therefore, by Lemma \ref{lemma:empto} and \eqref{help:11}, \eqref{help:12}, we have
\begin{align}
 \sum_{i=1}^{n-1} \EE_{z' \sim p_i}(b_n(z') - b^*(z'))^2 &\leq 8 \sum_{z \in \cD_n} (b_n(z) - b^*(z))^2+ 4\log(4N/\delta),\notag\\
 & \leq 16 \sum_{z \in \cD_n}(b_n(z) - r)^2 - \min_{b' \in \cR}(b'(z) - r)^2 + O(\beta_{\cR})\notag \\
 & \leq O(\beta_{\cR}),\notag
\end{align}
where the second and third line hold due to the selection of $\beta_{\cR}$. 
Similarily, we have 
\begin{align}
     \sum_{i=1}^{n-1} \EE_{z' \sim p_i}\|f_n(z') - f^*(z')\|^2 \leq O(\beta_{\cF}). \label{help:fff}
\end{align}
    Next we derive the following bound. Consider $n_1<n_2<\dots<n_l$ to be some $n\in[N]$ where $\cF_n$ gets updated. Then at some $n_i$, we have
    \begin{align}
        \sum_{x,u,y,r \in \cD_{n_{i+1}}}(f_{n_i}(x,u) - y)^2 \geq  \min_{f' \in \cF} \sum_{x,u,y,r \in \cD_{n_{i+1}}} (f'(x,u) - y)^2 + 5\beta_{\cF} \geq \sum_{x,u,y,r \in \cD_{n_i+1}} (f^*(x,u) - y)^2 + 4\beta_{\cF}. \notag
    \end{align}
    where the second inequality holds since $f^* \in \cF_{n_i+1}$ due to Theorem \ref{thm:2:regret}. Meanwhile, since $f_{n_i}$ is updated at $n_i$-th step, then $f_{n_i} \in \cF_{n_i}$, which is
    \begin{align}
                \sum_{x,u,y,r \in \cD_{n_i}}(f_{n_i}(x,u) - y)^2 \leq  \min_{f' \in \cF} \sum_{x,u,y,r \in \cD_{n_i}} (f'(x,u) - y)^2 + \beta_{\cF} \leq \sum_{x,u,y,r \in \cD_{n_i}}(f^*(x,u) - y)^2 + \beta_{\cF}.\notag
    \end{align}
    Therefore, we have
    \begin{align}
        \sum_{x,u,y,r \in \cD_{n_{i+1}}\setminus \cD_{n_{i}}}[ (f_{n_i}(x,u) - y)^2 - (f^*(x,u) - y)^2] \geq 3\beta_{\cF}. \notag
    \end{align}
    By Lemma \ref{lemma:empto}, we have
    \begin{align}
    3\beta_{\cF} \leq \sum_{x,u,y,r \in \cD_{n_{i+1}}\setminus \cD_{n_{i}}}[ (f_{n_i}(x,u) - y)^2 - (f^*(x,u) - y)^2] \leq \sum_{x,u,y,r \in \cD_{n_{i+1}}\setminus \cD_{n_{i}}}\frac{3}{2} (f_{n_i}(x,u) - f^*(x,u))^2 + \beta_{\cF},
    \end{align}
    which suggests that 
    \begin{align}
        \sum_{n=n_i}^{n_{i+1}-1} (f_{n}(z_n) - f^*(z_n))^2 \geq \beta_{\cF},\notag
    \end{align}
    where we use the fact that $f_n = f_{n_i}$ when $n_i \leq n<n_{i+1}$. Therefore, taking summation from $i = 1,\dots, l$, we have
    \begin{align}
        l\cdot \beta_{\cF}\leq \sum_{n=1}^{N} (f_{n}(z_n) - f^*(z_n))^2 = O\bigg(\sum_{n=1}^{N} \EE_{z \sim p_n}(f_{n}(z) - f^*(z))^2 + \beta_{\cF}\bigg) \leq O(d_{\cF}\beta_{\cF}\log N),\notag
    \end{align}
    where the first equality holds due to Lemma \ref{crhs1}, the second inequality holds due to \eqref{help:fff} and Lemma \ref{lemma:runzhe}. It suggests the switching number $l = O(d_{\cF}\log N)$. Similarily, the switching number of $\cR_n$ is also bounded by $O(d_{\cR}\log N)$. Combining them obtains the final result. 
\end{proof}

\section{Proof of Theorem \ref{thm:3}}\label{app:thm3}
The main idea of this proof originates from \citet{xiong2023general}. 
For the ease of presentation, we denote $p_{n,1},\dots, p_{n,m} = p_n$.
We divide episodes $n = 1,\dots, N/m$ into disjoint sets $E_j,j =  0,1,\dots, J$, where 
\begin{align}
&j = 0, n \in E_0: &\sum_{i=1}^m\EE_{z_{n,i} \sim p_{n,i}}\|f_{n}(z_{n,i}) - f^*(z_{n,i})\|^2 < 100C_{\cT}\cdot \beta_{\cF},\notag \\
    &j\geq 1, n \in E_j: &100C_{\cT}\cdot 2^{j-1}\beta_{\cF} \leq \sum_{i=1}^m\EE_{z_{n,i} \sim p_{n,i}}\|f_{n}(z_{n,i}) - f^*(z_{n,i})\|^2 < 100C_{\cT}\cdot 2^{j}\beta_{\cF}. \label{help:32}
\end{align}
Apparently, we have $J = O(\log N)$ since $f \leq 1$ and $m \leq N$. Meanwhile, note that by the definition of $f_n$, which is updated on $n$-th episode, then we have
\begin{align}
    &\sum_{n' = 1}^{n-1}\sum_{i=1}^m\EE_{z_{n',i} \sim p_{n',i}}\|f_{n}(z_{n',i}) - f^*(z_{n',i})\|^2 \notag \\
    & \leq C_{\cT}\sum_{n' = 1}^{n-1}\sum_{i=1}^m\EE_{z_{n',i} \sim \mathbb{P}_{\cT,\pi_{n'}, q_{n'}}(\cdot|z_{n',i-1},\dots, z_{n',1})}\bigg[\|f_{n}(z_{n',i}) - f^*(z_{n',i})\|^2\bigg]\notag \\
    & \leq C_{\cT}\bigg[4\sum_{n' = 1}^{n-1}\sum_{i=1}^m\|f_{n}(z_{n',t_{n',i}}) - f^*(z_{n',t_{n',i}})\|^2 + \beta_{\cF}\bigg]\notag \\
    & \leq 100 C_{\cT}\beta_{\cF},\label{help:33}
\end{align}
where the first inequality due to the definition of $C_{\cT}$, the second one holds due to Lemma \ref{crhs} and the selection of $\beta_{\cF}$, the last one holds due to Lemma \ref{lemma:empto} and the selection of $\beta_{\cF}$. Combining the upper bound in \eqref{help:32} and \eqref{help:33}, we have
\begin{align}
    \forall j \geq 0, \forall n\in E_j, \sum_{i=1}^m\EE_{z_{n,i} \sim p_{n,i}}\|f_{n}(z_{n,i}) - f^*(z_{n,i})\|^2 + \sum_{n' = 1}^{n-1}\sum_{i=1}^m\EE_{z_{n',i} \sim p_{n',i}}\|f_{n}(z_{n',i}) - f^*(z_{n',i})\|^2 \leq 200C_{\cT}\cdot2^{j}\beta_{\cF},\notag
\end{align}
therefore, by Lemma \ref{lemma:runzhe}, we have
\begin{align}
    \sum_{i=1}^m\EE_{z_{n,i} \sim p_{n,i}}\|f_{n}(z_{n,i}) - f^*(z_{n,i})\|^2 + \sum_{n' = 1}^{n-1}\sum_{i=1}^m\EE_{z_{n',i} \sim p_{n',i}}\|f_{n'}(z_{n',i}) - f^*(z_{n',i})\|^2 \leq O(d_{\cF}C_{\cT}\cdot2^{j}\beta_{\cF}\log N),\label{help:35}
\end{align}
Next, for $j \geq 1$, we bound \eqref{help:35} from another direction. We have
\begin{align}
    &\sum_{i=1}^m\EE_{z_{n,i} \sim p_{n,i}}\|f_{n}(z_{n,i}) - f^*(z_{n,i})\|^2 + \sum_{n' = 1}^{n-1}\sum_{i=1}^m\EE_{z_{n',i} \sim p_{n',i}}\|f_{n'}(z_{n',i}) - f^*(z_{n',i})\|^2\notag \\
    & \geq \sum_{n' \in E_j, n'<n}\sum_{i=1}^m\EE_{z_{n',i} \sim p_{n',i}}\|f_{n'}(z_{n',i}) - f^*(z_{n',i})\|^2\notag \\
    & \geq |\{n' \in E_j, n'<n\}|\cdot 100C_{\cT}\cdot 2^{j-1}\beta_{\cF}, \label{help:36}
\end{align}
where the second inequality holds due to the definition of $E_j$. Therefore, combining \eqref{help:35} and \eqref{help:36} and setting $n$ to be the max element in $E_j$, we have $|E_j| \leq O(d_{\cF}\log N)$ for all $j \geq 1$. Similarily, for the reward function $b$, we define $F_j$ similar to $E_j$, we can also obtain that 
\begin{small}
    \begin{align}
    &j=0, \forall n \in F_0: &\sum_{i=1}^m\EE_{z_{n,i} \sim p_{n,i}}(b_{n}(z_{n,i}) - b^*(z_{n,i}))^2 + \sum_{n' = 1}^{n-1}\sum_{i=1}^m\EE_{z_{n',i} \sim p_{n',i}}(b_{n'}(z_{n',i}) - b^*(z_{n',i}))^2 \leq O(d_{\cR}C_{\cT}\beta_{\cR}\log N)\notag \\
    &j\geq 1, \forall n \in F_j:&|F_j| \leq O(d_{\cR}\log N).\label{help:37}
\end{align}
\end{small}

Finally we bound the final regret. We look at the bound of the suboptimality gap $R_{n,i}$ from \eqref{help:boundz}, where
\begin{align}
    R_{n,i} &\leq  LT \sqrt{2T \exp(K T)}\sqrt{A_{n,i}} + T\sqrt{B_{n,i}}, \  A_{n,i}: = \EE_{z\sim p_{n,i}}\|f^*(z) - f_{n,i}(z)\|_2^2 ,
    B_{n,i}: = \EE_{z\sim p_{n,i}} |b_{n,i}(z)- b^*(z)|^2.\label{help:boundz11}
\end{align}
Then for the total regret, we have
\begin{align}
    \sum_{i=1}^{m}\sum_{n=1}^{N/m} R_{n,i}
    & = \sum_{i=1}^{m}\bigg(\sum_{n \in E_0\cap F_0}R_{n,i} + \sum_{n \notin E_0\cap F_0}R_{n,i}\bigg)\notag \\
    & \leq \sum_{i=1}^{m}\sum_{n \in E_0\cap F_0}LT \sqrt{2T \exp(K T)}\sqrt{A_{n,i}} + T\sqrt{B_{n,i}} + (|E_1|+\dots+|F_1|+\dots)\cdot T\notag \\
    & \leq LT \sqrt{2T \exp(K T)}\sqrt{|E_0\cap F_0|\sum A_{n,i}} + T\sqrt{|E_0\cap F_0| \sum B_{n,i}} + mT\log N\cdot O(d_{\cF} + d_{\cR})\notag \\
    & \leq O\big(LT \sqrt{2T \exp(K T)}\sqrt{N d_{\cF}C_{\cT}\beta_{\cF}\log N} + T\sqrt{N d_{\cR}C_{\cT}\beta_{\cR}\log N} + mT(d_{\cF} + d_{\cR})\log N\big),\notag
\end{align}
where the first inequality holds due to \eqref{help:boundz11} and the fact $R_{n,i} \leq T$, the second one holds due to Cauchy-Schwarz inequality, the last one holds due to conditions in \eqref{help:35}, \eqref{help:37} and applying them to Lemma \ref{lemma:runzhe}. Therefore, our proof holds.

\section{Technical Lemmas}
\begin{lemma}[Gronwall's Inequality \citep{10.1215/S0012-7094-43-01059-2}]
Let $u(t)$ be a non-negative, continuous function on the interval $[a, b]$. Suppose that
\[
u(t) \leq K + \int_a^t \gamma(s) u(s) \, ds
\]
for all $t \in [a, b]$, where $K$ is a non-negative constant and $\gamma(s)$ is a non-negative, continuous function on $[a, b]$. Then,
\[
u(t) \leq K \exp\left( \int_a^t \gamma(s) \, ds \right)
\]
for all $t \in [a, b]$.
\end{lemma}

\begin{lemma}[\citealt{zhang2021improved}]
Let \( (\mathcal{F}_i)_{i \geq 0} \) be a filtration. Let \( (X_i)_{i \geq 1} \) be a sequence of random variables such that \( |X_i| \leq 1 \) almost surely, and \( X_i \) is \( \mathcal{F}_i \)-measurable. For every \( \delta \in (0,1) \), we have
\begin{align}
    &\Pr\left( \sum_{i=1}^n \mathbb{E}\left[ X_i^2 \mid \mathcal{F}_{i-1} \right] \geq 8\sum_{i=1}^n X_i^2 + 4 \ln \left( \frac{4}{\delta} \right) \right) \leq (\log(n)+1)\delta.\notag
\end{align}

\label{crhs}
\end{lemma}
\begin{lemma}[\citealt{zhang2021improved}]
Let \( (\mathcal{F}_i)_{i \geq 0} \) be a filtration. Let \( (X_i)_{i \geq 1} \) be a sequence of random variables such that \( |X_i| \leq 1 \) almost surely, and \( X_i \) is \( \mathcal{F}_i \)-measurable. For every \( \delta \in (0,1) \), we have
\begin{align}
    &\Pr\left( \sum_{i=1}^n X_i^2 \geq 8\sum_{i=1}^n \mathbb{E}\left[ X_i^2 \mid \mathcal{F}_{i-1} \right] + 4 \ln \left( \frac{4}{\delta} \right) \right) \leq (\log(n)+1)\delta.\notag
\end{align}

\label{crhs1}
\end{lemma}

\section{Additional Details of Experiments for Diffusion Model Fine-Tuning} \label{app:seiko_exp}
In this section we introduce additional experiment details in Section \ref{exp:PURE-SEIKO}. 

\subsection{From Theory to Practice}
$\algseiko$ offers the first concrete realization of the general update schemes in Algorithms~\ref{alg:alg2} and \ref{alg:alg3}.

\paragraph{How the Theoretical Insights Inform the Design of $\algseiko$}

Intuitively, Theorem \ref{thm:2} suggests that by updating the policy and initial distribution less frequently, as prescribed by Algorithm~\ref{alg:alg2}, we can still maintain a high-probability confidence set for both the dynamics and the reward; Theorem \ref{thm:3} indicates that following Algorithm \ref{alg:alg3}, performing multiple measurements within each episode--while keeping the \emph{total} number of measurements unchanged--can yield comparable results to more rollout baselines.

As mentioned in the main context, \textsc{SEIKO} already adopts a low-switching schedule: training is divided into $\mathcal K$ batches with geometrically increasing sizes, $B_{i+1} = \eta_{\text{base}} B_i$ for $i \in [\mathcal K]$.  Nevertheless, diffusion-model fine-tuning under \textsc{SEIKO} remains slow. Guided by Theorem \ref{thm:2} and Theorem \ref{thm:3}, we insert extra mid-episode measurements to speed up data collection without enlarging the sample budget, producing Algorithm~\ref{alg:pure-seiko}, the $\algseiko$ variant.

\paragraph{How the Experiment Result Backup the Theoretical Results}

Figure~\ref{fig:seiko-base} demonstrates that adding in-trajectory measurements markedly shortens sampling time while achieving aesthetic scores comparable to the original \textsc{SEIKO}. This empirical behavior substantiates the prediction of Theorem~\ref{thm:3}.

\subsection{Details of \textsc{SEIKO}}
Progress in \textsc{SEIKO} is primarily evaluated using a pre-trained aesthetic scorer, specifically the LAION Aesthetics Predictor V2~\citep{schuhmann2022laion}. Following \citet{uehara2024feedback}, we fix the total number of scorer evaluations (i.e., measurements) at $N = 19200$. To address uncertainty in reward estimation, \citet{uehara2024feedback} propose two versions of the uncertainty oracle: \emph{Bootstrap} (bootstrapped neural networks) and \emph{UCB} (an uncertainty estimate derived from the network’s last layer). We adopt the UCB variant, as it generally produces superior aesthetic scores. For the backbone diffusion model, \textsc{SEIKO} employs Stable Diffusion V1.5~\citep{rombach2022high}, which we also adopt as our pre-trained model. While Stable Diffusion V1.5 was originally trained with $1000$ discretized denoising steps, we follow \textsc{SEIKO} and reduce it to $50$ steps at inference time for improved sampling efficiency. For notational simplicity, we define the denoising time from $0$ (fully denoised) to $T$ (initial noise), inverting the conventional $T$-to-$0$ timeline.

\subsection{$\algseiko$ Algorithm}

Building on \textsc{SEIKO}, we introduce a more flexible framework, $\algname$, which incorporates multiple in-trajectory measurements and allows control over the frequency of policy updates. We refer to this specialized version as $\algseiko$, whose pseudo-code is presented in Algorithm~\ref{alg:pure-seiko}. In brief, at episode $n$, we begin from an initial state $x(t_0) \sim q_{n}$ and simulate the trajectory using the following update rule with time step $\Delta t$:
$$
x(t_k) = x(t_{k-1}) + f_{n}(t_{k-1}, x(t_{k-1})) \, \Delta t + g^*(x(t_{k})) \, (\Delta w(t_k)),
$$
where $\Delta w(t_k) \sim \mathcal{N}(0, (\Delta t)^2 \cdot I)$, $t_{k} = t_{k-1}+\Delta t$. The trajectory is then used to compute a Riemann sum over intermediate values of $b_n$—a learned reward function—to approximate the cumulative reward $R$. The dataset $\mathcal{D}_n$ for training $b_n$ is updated continuously across episodes, which ensures that $b_n$ converges toward the true reward function $b^*$ over time. This approach is commonly used for approximating integrals in diffusion models \citep{uehara2024feedback}.

To optimize $R = R(\pi, q, f, b)$ over the confidence sets, we construct upper confidence bounds (UCBs) for $f$ and $b$ based on their respective confidence sets $\mathcal{F}_n$ and $\mathcal{R}_n$. Then, we jointly optimize $R$ over $\pi$ and $q$ with UCBs of $f_n, b_n$ described above.

In this framework, $B_1$ denotes the batch size in the first outer loop, while the hyperparameter $\eta_{\text{base}}$ determines the growth factor for the number of samples in subsequent outer loops, following the relation $B_{i+1} = \eta_{\text{base}} \cdot B_i$.

\begin{algorithm*}
\caption{\textbf{P}olicy \textbf{U}pdate and \textbf{R}olling \textbf{E}fficient CTRL for Optimi\textbf{S}tic fin\textbf{E}-tuning of d\textbf{I}ffusion with \textbf{K}L c\textbf{O}nstraint ($\algseiko$)}
\label{alg:pure-seiko}
\begin{algorithmic}[1]
\REQUIRE Total measurement number $N$, initial distribution class $q\in \cQ$, pre-trained drift class $f^{\text{pre}} \in \cF$, diffusion term $g^*$, ground-truth reward $r\in \cR$, reward approximation $\hat r$, episode length $T$, sampler $\cT$, diffusion hyperparameter $\alpha, \{\beta_n\}\in \mathbb R^+$, counter $\kappa$, measurement frequency $m$, initial batch size $B_1 \in \mathbb Z^+$, hyperparameter $\eta_{\text{base}} \in \mathbb R^+$.

\STATE   Initialize $f_1=f^{\text{pre}}, q_1 = q, \kappa=0$.
\FOR{episode $n = 1,\dots, \lfloor N/m\rfloor$}
\STATE Sample $t_{n,1},\cdots, t_{n,m}\sim \cT, t_{n,0}=0$
\STATE Execute $dx(t) = f_{n-1}(t,x(t))dt + g^*(x(t))dw(t), x(0)\sim q_{n-1}$, receive feedback $y_{n,i} = r(x(t_{n,i})) + \epsilon$
\STATE Update $\cD_{n+1}\leftarrow \cD_{n} \cup (\{x(t_{n,i}),y_{n,i}\}_{i=1}^m)$.
\STATE Set $\hat r_{n+1} \leftarrow \hat r_{n}$, $f_{n+1} \leftarrow f_{n}$
\\ \texttt{// If collected enough samples, update the reward and diffusion once.}
\STATE \textbf{if}\ $\frac{\eta_{\text{base}}^{\kappa}\cdot B_1}{m}\leq n$ \textbf{then}
\STATE \quad Train $\hat r_{n+1}$ on $\cD_{n+1}$, and update $f_{n+1}$, $q_{n+1}$ by solving 
\begin{align*} f_{n+1}, q_{n+1} &= \argmax_{f\in \cF, q\in\cQ}\;\mathbb E_{\mathbb P^{f,q}}[\hat r(x(T))] \;-\;\alpha\,\mathrm{KL}(\mathbb P^{f,q}|\mathbb P^{f_1, q_1}) \;-\;\beta_{n} \,\mathrm{KL}(\mathbb P^{f,q}\|\mathbb P^{f_{n},q_{n}}), 
\end{align*} 
using the DDIM optimizer~\citep{song2020denoising}, where $\mathbb P^{f,q}\in \Delta(\mathcal{X})$ refers to the marginal distribution at $T$. Then set $\kappa\leftarrow \kappa+1$. 

\ENDFOR
\STATE \textbf{Output: } $f_{\lfloor\frac{N}{m}\rfloor+1},q_{\lfloor\frac{N}{m}\rfloor+1}$
\end{algorithmic}
\end{algorithm*}

\begin{figure*}[t]
    \centering
    \includegraphics[width=0.8\textwidth]{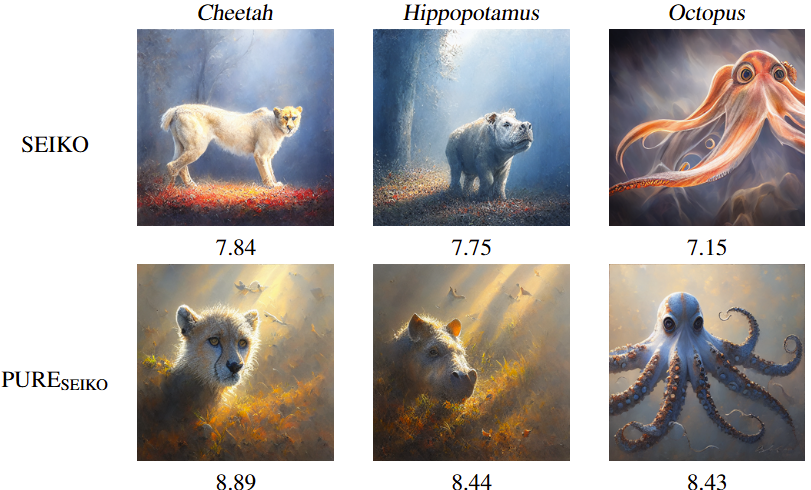}
    \caption{Qualitative comparison between \textsc{SEIKO} and our $\algseiko$ approach, with aesthetic scores listed below each image.}
    \label{fig:comparison}
\end{figure*}

\subsection{Additional Experiment Results}
Figure~\ref{fig:comparison} presents a qualitative comparison between samples generated by the diffusion model fine-tuned with \textsc{SEIKO} and our proposed $\algseiko$. Notably, $\algseiko$ achieves a comparable output image quality to \textsc{SEIKO} while requiring fewer computational resources.

\subsection{$\algseiko$ Experiment Parameters}

\paragraph{Prompts} 
For a fair comparison with the \textsc{SEIKO} algorithm, we follow the prompt settings from \citet{uehara2024feedback}'s image task for both training and evaluation. Specifically, the training phase utilizes prompts from a predefined list of 50 animals~\citep{black2023training,prabhudesai2023aligning}, while the evaluation phase employs the following unseen animal prompts: snail, hippopotamus, cheetah, crocodile, lobster, and octopus.

\paragraph{Hyperparameters} 
Table~\ref{table:hyperparameter-seiko} summarizes the key hyperparameters for fine-tuning. We use ADAM~\citep{kingma2014adam} as the optimizer.

\begin{table}[htbp]
    \caption{Important hyperparameters for fine-tuning.}
    \centering
    \begin{tabular}{l l c}
        \toprule
        \textbf{Method} & \textbf{Type} &  \\
        \midrule
        \textsc{SEIKO} & Batch size & 128 \\
         & KL parameter $\beta$ & 0.01 \\
         & UCB parameter $C_1$ & 0.002 \\
         & Sampling to neural SDE & Euler \\
         & Step size (fine-tuning) & 50 \\
         & Epochs (fine-tuning) & 100 \\
         \midrule
         $\algseiko$ & $\lambda$ (temperature parameter in \ref{exp:seiko-implement}) & 6  \\ 
         & $N$ (total measurement number) & 19200 \\
         & $m$ (measurement frequency) & 4 \\
         & $B_1$ (number of samples in the first outer loop) & 1280 \\
         & $\eta_{\text{base}}$ (growth factor for subsequent outer loop) & 2 \\
        \bottomrule
    \label{table:hyperparameter-seiko}
    \end{tabular}
\end{table}

\section{continuous-time control Experiments} \label{app:oderl_exp}

In this section we introduce additional experiment details about continuous-time control tasks. 

\subsection{From Theory to Practice}
$\algenode$ offers the second realization of the general update schemes in Algorithms~\ref{alg:alg2} and \ref{alg:alg3}.

\paragraph{How the Theoretical Insights Inform the Design of $\algenode$}

Building on \textsc{ENODE}, we introduce a more flexible framework, $\algname$, which enables control over the frequency of policy updates. A specialized instance of this framework, referred to as $\algenode$, is detailed in Algorithm~\ref{alg:pure-enode}. As noted in the main text, \textsc{ENODE} already adopts a low-rollout strategy. In $\algenode$, we further incorporate a batch-style update scheme inspired by Algorithm~\ref{alg:alg2}. Specifically, following a scheme similar to \textsc{SEIKO}, the batch size $B_i$ doubles at each step according to $B_{i+1} = 2B_i$, while keeping the total sample budget $N$ fixed.

\paragraph{How the Experiment Result Backup the Theoretical Results}

Figure~\ref{fig:oderl} demonstrates that reducing the number of policy updates can significantly shorten training time while maintaining comparable rewards. This empirical observation supports the insight of Theorem~\ref{thm:2}.

\subsection{Additional Details of Experiments for continuous-time control tasks}\label{app:oderl_env}

We conduct experiments on the Acrobot, Pendulum, and Cart Pole tasks using the OpenAI Gym simulator~\citep{brockman2016openai}. In all tasks, the system begins in a hanging-down state, and the objective is to swing up and stabilize the pole(s) in an upright position~\citep{yildiz2021continuous}. We put related parameters in Table \ref{tab:environment_specifications}. 
\begin{table}[h!]
    \centering
    \caption{Environment specifications}
    \label{tab:environment_specifications}
    \begin{tabular}{lccccc}
        \toprule
        \textbf{Environment} & \(c_p\) & \(c_a\) & \(\alpha_{\max}\) & \(\mathbf{s}^{\text{box}}\) & \(\mathbf{s}^{\text{goal}}\) \\
        \midrule
        \textbf{Acrobot} & \(1\text{e-}4\) & \(1\text{e-}2\) & 4 & \([0.1, 0.1, 0.1, 0.1]\) & \([0, 2\ell]\) \\
        \textbf{Pendulum} & \(1\text{e-}2\) & \(1\text{e-}2\) & 2 & \([\pi, 3]\) & \([0, \ell]\) \\
        \textbf{Cart Pole} & \(1\text{e-}2\) & \(1\text{e-}2\) & 3 & \([0.05, 0.05, 0.05, 0.05]\) & \([0, 0, \ell]\) \\
        \bottomrule
    \end{tabular}
\end{table}

\paragraph{Acrobot} 
The Acrobot system consists of two links connected in series, forming a chain with one end fixed. The joint between the two links is actuated, and the goal is to apply torques to this joint to swing the free end above a target height, starting from the initial hanging-down state. We use the fully actuated version of the Acrobot environment, as no method has successfully solved the underactuated balancing problem, consistent with \citet{zhong2020unsupervised}. The control space is discrete and deterministic, representing the torque applied to the actuated joint. The state space consists of the two rotational joint angles and their angular velocities.

\paragraph{Pendulum} 
The inverted pendulum swing-up problem is a fundamental challenge in control theory. The system consists of a pendulum attached at one end to a fixed pivot, with the other end free to move. Starting from a hanging-down position, the goal is to apply torque to swing the pendulum into an upright position, aligning its center of gravity directly above the pivot. The control space represents the torque applied to the free end, while the state space includes the pendulum’s x-y coordinates and angular velocity.

\paragraph{Cart Pole} 
The Cart Pole system comprises a pole attached via an unactuated joint to a cart that moves along a frictionless track. Initially, the pole is in an upright position, and the objective is to maintain balance by applying forces to the cart in either the left or right direction. The control space determines the direction of the fixed force applied to the cart. The state space includes the cart’s position and velocity, as well as the pole’s angle and angular velocity.

\paragraph{Initial State} 
In all environments, the initial position \(\mathbf{q}(0)\) is uniformly distributed as:
\[
\mathbf{q}(0) \sim \text{Unif}\left(-\mathbf{s}^{\text{box}}, \mathbf{s}^{\text{box}}\right),
\]
where $\mathbf{s}^{\text{box}}$ is the position parameter.
\paragraph{Reward Functions}
For all three tasks, we denote the state by \(x = (\mathbf{q}, \mathbf{p})\), where $\mathbf{q}$ denotes the position and $\mathbf{p}$ denotes the velocity (momentum). 
Given a state $x = (\mathbf{q}, \mathbf{p})$ and a control unit $u$, the differentiable reward function is defined as:
\[
b(x, u) = \exp\left(-\|\mathbf{q} - \mathbf{s}^{\text{goal}}\|_2^2 - c_p \|\mathbf{p}\|_2^2 - c_a \|u\|_2^2\right),
\]
where $\mathbf{s}^{\text{goal}}$ denotes the goal position, \(c_p\) and \(c_a\) denote environment-specific constants. The exponential formulation ensures that the reward remains within \([0, 1]\), while penalizing deviations from the target state and excessive control effort.
The environment-specific parameters are set following the exact configurations in \citet{yildiz2021continuous}.

\paragraph{Baseline}
We highlight several unique components of $\algenode$. First, \textsc{ENODE} trains dynamics following an evidence lower bound (ELBO) \citep{blei2017variational} setup, which aims to minimize the negative log-likelihood function between the true state and the imagined state generated by the dynamics function. This process can be regarded as an approximation of our introduced measurement oracle. $\algenode$ employs a sampler that generates time steps consisting of several mini-batches, where each mini-batch comprises consecutive time steps with a randomly selected initial time step. To train the optimal policy, \textsc{ENODE} adopts the standard actor-critic framework based on the learned dynamics. Further details can be found in \citet{yildiz2021continuous}.

\paragraph{Neural Network Architectures}
We adopt the same neural network architecture as described in \citet{yildiz2021continuous}. The dynamics, actor, and critic functions are approximated using multi-layer perceptrons (MLPs). The same neural network architectures were employed across all methods and environments, as detailed below:

\begin{itemize}[leftmargin = *]
    \item \textbf{Dynamics:} The dynamics function is modeled with three hidden layers, each containing 200 neurons, utilizing Exponential Linear Unit (ELU) activations. Experimental observations suggest that ELU activations enhance extrapolation on test sequences. 
    \item \textbf{Actor:} The actor network consists of two hidden layers, each with 200 neurons, using ReLU activations. This design is motivated by the observation that optimal policies can often be approximated as a collection of piecewise linear functions. The final output of the network is passed through a tanh activation function and scaled by $\alpha_{\max}$.
    \item \textbf{Critic:} The critic network also consists of two hidden layers, each with 200 neurons, but employs tanh activations. Since state-value functions must exhibit smoothness, tanh activations are more suitable compared to other activation functions. Empirical results indicate that critic networks with ReLU activations tend to overfit to training data, leading to instability and degraded performance when extrapolating beyond the training distribution.
\end{itemize}

\begin{algorithm*}
\caption{\textbf{P}olicy \textbf{U}pdate and \textbf{R}olling \textbf{E}fficient CTRL with \textbf{E}nsemble \textbf{N}eural \textbf{ODE}s ($\algenode$)}
\label{alg:pure-enode}
\begin{algorithmic}[1]
\REQUIRE Total measurement number $N$, measurement frequency $m$, episode length $T=50$, sub-episode length $T'=5$, true reward $r \in \cR$, dynamic class $\cF$, policy class $\Pi$, initial batch size $B_1$, number of initial trajectories to collect $\eta_{\text{init}} \in \mathbb{Z}^+$, counter $\kappa$, sampler $\cT$, mini-batch size $N_d = 5$, time grid $\Delta t \in \mathbb R^+$, hyperparameter $\eta_{\text{base}}\in\mathbb R^+$

\STATE Initialize dynamic $f$, policy $\pi$ as untrained Neural Network. Initialize an initial measurement dataset $\cD_0 = \{\{    x(t_{i,j}), u(t_{i,j}) \}_{j=1}^m\}_{i=1}^{\eta_{\text{init}}}$, collecting $\eta_{\text{init}}$ trajectories with smooth random policies defined in \citet{yildiz2021continuous}; $\kappa=0$ 
\FOR{episode $n = \eta_{\text{init}}+1,\dots, \lfloor N/m\rfloor$}

\STATE Run sampler $\cT$ and receive $t_{n,i}, i\in [m = N_dT'/\Delta t]$, where $\{t_{n,i}\}$ consists of $N_d$ number of independent mini-batches, each mini-batch consists of $t^0, t^0+\Delta t,\dots, t^0+T'$ consequent time steps with grid $\Delta t$. 
\STATE Execute policy $\pi$ and observe at $t_{n,i}$. Update dataset $\mathcal D_{n+1}\leftarrow \mathcal{D}_n \cup \{ x(t_{n,i}), u(t_{n,i})\}_{i=1}^m$
\\ \texttt{// If collected enough samples, update the dynamic and actor-critic once.}
\IF{$\frac{\eta_{\text{base}}^{\kappa}\cdot B_1}{m}\leq n - \eta_{\text{init}}$}

        \STATE Train $f$ by using the ELBO on $\mathcal{D}_{n+1}$

        \STATE Train $\pi$ following the actor-critic schedule based on the dynamic $f$ following Algorithm 1 in \citet{yildiz2021continuous}

    \STATE Set $\kappa\leftarrow \kappa+1$. 
    \ENDIF
\ENDFOR
\STATE \textbf{Output: } Policy $\pi$
\end{algorithmic}
\end{algorithm*}

\subsection{Continuous-time Control Experiment Details}\label{app:oderl_exp_details}

\paragraph{Additional Details}
We include the batch size information in Table \ref{tab:policy_updates}. In all experiments, we use \textsc{dopri5 (RK45)} as the adaptive ODE solver, as suggested by \citet{yildiz2021continuous}. We use the ADAM optimizer \citep{kingma2014adam} to train all model components, with the learning rate varying by environment.

\begin{table}[h!]
    \centering
    \caption{Data and Policy Updates}
    \label{tab:policy_updates}
    \begin{tabular}{l|c|cc|cc}
        \toprule
        \textbf{Environment} & Model &  $N/m$ & $\eta_{\text{init}}$ & Number of Batches & Batch Sizes \\
        \hline
        \multirow{ 2}{*}{\textbf{Acrobot}} & \textsc{ENODE} & \multirow{ 2}{*}{$87$} & \multirow{ 2}{*}{$7$} & $20$ & $[4,\dots,4]$ with length 20\\
        & $\algenode$ & & & $6$ & $[2,4,8,16,18,32]$\\
        \multirow{ 2}{*}{\textbf{Pendulum}} & \textsc{ENODE} & \multirow{ 2}{*}{$9$} & \multirow{ 2}{*}{$3$} & $6$ & $[1,1,1,1,1,1]$ \\
        & $\algenode$ & & & $3$ & $[1,2,3]$\\
        \multirow{ 2}{*}{\textbf{Cart Pole}} & \textsc{ENODE} & \multirow{ 2}{*}{$80$} & \multirow{ 2}{*}{$5$} & $25$ & $[3,\dots,3]$ with length 25 \\
        & $\algenode$ & & & $6$ & $[2,4,8,13,16,32]$\\
        \bottomrule
    \end{tabular}
\end{table}

\newpage
\subsection{Ablation Study}\label{app:oderl_ablation_study}
For simplicity of presentation, we use $N_{\text{pu}}$ to denote the number of batches where the dynamics and policy are only updated at the beginning of each batch. For all experiments in the ablation study of continuous-time control, we select the Acrobot environment, as it requires a moderate amount of time to reach success.

\subsubsection{Varying Number of Batches $N_\text{pu}$}

First, we investigate the impact of different batch update scheduling strategies, namely, the policy update times $N_\text{pu}$. In addition to the doubling strategy \( \algenode^{N_\text{pu}=6} \) introduced in Section~\ref{exp:oderl}, we implement two alternative variations of \( \algenode \): \textbf{(a)} \( \algenode^{N_\text{pu}=10} \), which maintains a constant batch size \( B_i \) at each step (equaling strategy) but reduces the policy update frequency to half of the original \textsc{ENODE}, and \textbf{(b)} a more aggressive tripling approach \( \algenode^{N_\text{pu}=4} \), where the batch size \( B_i \) triples at each step, following the rule \( B_{i+1} = 3B_i \). In all cases, we ensured that the total sample budget $N$ remained consistent, and the total episode number is \( N/m = 87 \), with each data trajectory containing \( m = 250 \) observations to align with the main experimental setup. Further details are provided in Table~\ref{tab:enode-ablation1}.  

The results, shown in Figure~\ref{fig:enode-ablation1}, indicate that overall runtime decreases as the number of policy updates \( N_\text{pu} \) is reduced. However, the batch update scheduling strategy plays a crucial role in determining the efficiency of policy learning. For instance, although \( \algenode^{N_\text{pu}=10} \) has a larger \( N_{\text{pu}} \) than \( \algenode^{N_\text{pu}=6} \) and might be expected to achieve higher average rewards, it frequently failed to meet the success criteria after exhausting the batch update scheduler in several experiments. We attribute this phenomenon to the importance of ensuring that high-quality samples dominate the dataset in the later stages of policy updates. If initial low-quality samples remain prevalent, the agent may struggle to fully leverage the high-quality samples for effective learning. Conversely, an overly aggressive approach with very few policy updates, as in \( \algenode^{N_\text{pu}=4} \), leads to difficulties in processing the large influx of new trajectories in later stages of policy updates, resulting in unstable final performance.

\begin{table}[h!]
    \centering
    \caption{Ablation Study: Batch Update Scheduling Strategy $N_\text{pu}$}
    \label{tab:enode-ablation1}
    \begin{tabular}{lccc}
        \toprule
        Model & Strategy & $N_\text{pu}$ & $N_{\text{inc}}$ \\
        \midrule
        $\algenode^{N_\text{pu}=4}$ & Tripling & $4$ & $[2, 6, 18, 54]$ \\
        $\algenode^{N_\text{pu}=6}$ & Doubling & $6$ & $[2, 4, 8, 16, 18, 32]$ \\
        $\algenode^{N_\text{pu}=10}$ & Equaling & $10$ & $[4,\ldots, 4]$ with length 10 \\
        \bottomrule
    \end{tabular}
\end{table}

\subsubsection{Varying Number of Measurements $m$}
Similar to the ablation study in fine-tuning diffusion models (Section~\ref{exp:seiko-implement}), we analyze the impact of the number of measurements (\( m \)) on effective policy learning. In the Acrobot environment, the total number of measurements for one policy update is given by \( m =N_d \times \frac{t_s}{\Delta t}=5\times \frac{5}{0.1}= 250 \). In addition to \( m = 250 \), we evaluate alternative settings with \( m = 125, 500, \) and \( 1000 \) for trajectories included in \( \mathcal{D} \). Details of the modified parameters for each setting are provided in Table~\ref{tab:enode-ablation2}.  

The results, shown in Figure~\ref{fig:enode-ablation2}, indicate that increasing \( m \) can slightly reduce total training time by decreasing the number of required trajectories. However, unlike in diffusion model experiments, where data trajectory generation is the primary computational bottleneck, the bottleneck in continuous-time control experiments lies in the policy iteration process. As a result, the reduction in training time is relatively small compared to our \textsc{SEIKO} experiments. Moreover, excessively large values of \( m \) can negatively impact model performance, yielding low final rewards when $m=500$ and $1000$. This highlights the necessity of selecting an optimal \( m \) to balance solving continuous-time control problems effectively while maintaining training efficiency. These findings align with our claim in Theorem~\ref{thm:3}, reinforcing the importance of appropriately choosing \( m \) to achieve the best trade-off.

\begin{table}[h!]
    \centering
    \caption{Ablation Study: Number of Measurements $m$}
    \label{tab:enode-ablation2}
    \begin{tabular}{lcc}
        \toprule
        Model & $m$ & $N_{\text{inc}}$ \\
        \midrule
        $\algenode^{m=125}$ & $125$ & $[4, 8, 16, 32, 36, 64]$ \\
        $\algenode^{m=250}$ & $250$ & $[2, 4, 8, 16, 18, 32]$ \\
        $\algenode^{m=500}$ & $500$ & $[1,2,4,8,9,16]$ \\
        $\algenode^{m=1000}$ & $1000$ & $[1,1,2,4,4,8]$ \\
        \bottomrule
    \end{tabular}
\end{table}

\begin{figure*}[!htb]
\centering
\subfloat[\label{fig:enode-ablation1} $\algenode$ with varying $N_\text{pu}$] {\includegraphics[width=0.45\textwidth]{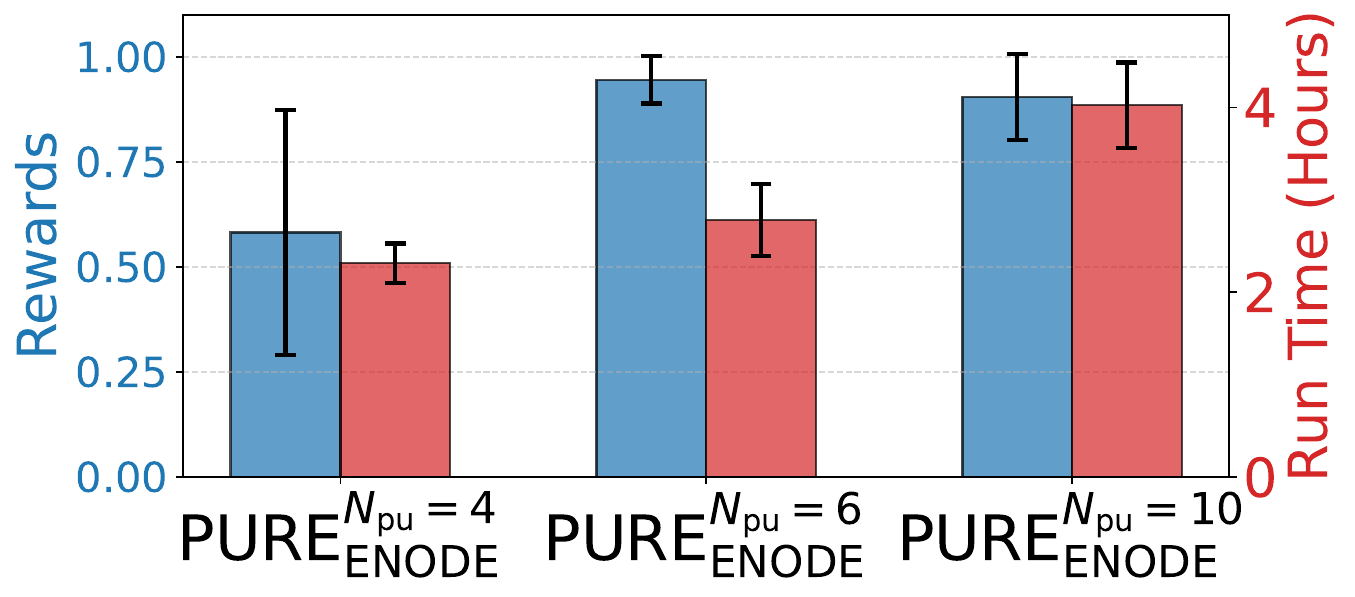}}\hfill
\subfloat[\label{fig:enode-ablation2} $\algenode$ with varying $m$]{\includegraphics[width=0.45\textwidth]{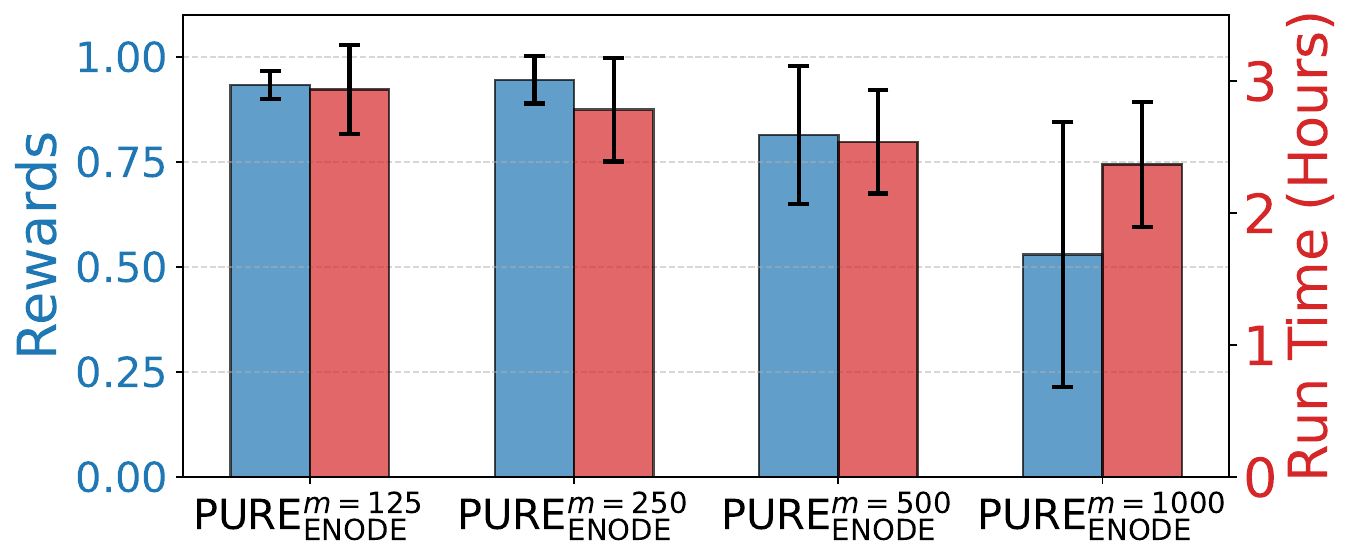}}
\caption{Summary of the ablation studies for continuous-time control in the Acrobot environment. Figures \ref{fig:enode-ablation1} and \ref{fig:enode-ablation2} analyze the impact of the number of policy updates $N_\text{pu}$ and the number of observations $m$ on the final rewards, respectively, considering either exhausting the scheduler or achieving success, whichever occurs first.} \label{fig:enode-ablation}
\vspace{-0.5cm}
\end{figure*}

\newpage

\end{document}